\documentclass{article}

\PassOptionsToPackage{table}{xcolor}

\usepackage[accepted]{icml2026}

\usepackage{microtype}
\usepackage{graphicx}
\usepackage{subcaption}  
\usepackage{booktabs}
\usepackage{amsmath,amssymb,amsthm,mathtools,bm}
\usepackage{algpseudocode}
\usepackage{enumitem}
\usepackage{longtable}
\usepackage{booktabs}
\usepackage{etoc}

\usepackage{hyperref}
\usepackage[capitalize,nameinlink]{cleveref}

\providecommand{\diag}{\operatorname{diag}}

\providecommand{\op}{\mathrm{op}}

\theoremstyle{plain}
\newtheorem{theorem}{Theorem}[section]
\newtheorem{lemma}[theorem]{Lemma}

\newtheorem{proposition}[theorem]{Proposition}

\theoremstyle{definition}
\newtheorem{definition}[theorem]{Definition}
\newtheorem{assumption}[theorem]{Assumption}

\theoremstyle{remark}
\newtheorem{remark}[theorem]{Remark}

\crefname{assumption}{Assumption}{Assumptions}
\Crefname{assumption}{Assumption}{Assumptions}

\newcommand{\newcustomtheorem}[2]{%
  \newenvironment{#1}[1]
  {%
   \par\noindent\textbf{#2~##1.}\itshape\space
  }
  {\par}
}

\newcustomtheorem{manualdefinition}{Definition}
\newcustomtheorem{manualremark}{Remark}
\newcustomtheorem{manualassumption}{Assumption}
\newcustomtheorem{manualproposition}{Proposition}
\newcustomtheorem{manualproperty}{Property}
\newcustomtheorem{manualcorollary}{Corollary}

\numberwithin{theorem}{section}
\numberwithin{equation}{section}

\DeclareMathOperator{\E}{\mathbb{E}}
\DeclareMathOperator{\Var}{Var}

\newcommand{\R}{\mathbb{R}}
\newcommand{\calO}{\mathcal{O}}
\newcommand{\tildeO}{\tilde{\mathcal{O}}}
\newcommand{\1}{\mathbf{1}}

\newcommand{\DeltaA}{\Delta(\mathcal{A})}

\DeclarePairedDelimiter{\norm}{\lVert}{\rVert}
\DeclarePairedDelimiter{\abs}{\lvert}{\rvert}
\DeclarePairedDelimiter{\ip}{\langle}{\rangle}

\newcommand{\vpi}{\bm{\pi}}
\newcommand{\vmu}{\bm{\mu}}

\icmltitlerunning{Second-Order Smooth Planning with Optimal-Transport Bellman Smoothing}

\begin{document}
\twocolumn[
\icmltitle{Second-Order Smooth Planning with Optimal-Transport Bellman Smoothing}

\icmlsetsymbol{equal}{*}

\begin{icmlauthorlist}
\icmlauthor{Tuan Dam}{1}
\end{icmlauthorlist}

\icmlaffiliation{1}{Hanoi University of Science and Technology, Hanoi, Vietnam}

\icmlcorrespondingauthor{Tuan Dam}{tuandq@soict.hust.edu.vn}

\icmlkeywords{Machine Learning, ICML}

\vskip 0.3in
]
\printAffiliationsAndNotice{}

\begin{abstract}
Planning with a generative model aims to estimate the value of a state using as few simulator calls as possible.
SmoothCruiser achieves problem-independent complexity $\widetilde O(\varepsilon^{-4})$ by exploiting the smoothness of the entropy-regularized Bellman backup, but its estimator is only first-order.
We show that the sample-complexity exponent of SmoothCruiser-type planners is governed by the order $\beta$ of the local Taylor remainder, giving oracle complexity $\widetilde O(\varepsilon^{-(2+2/(\beta-1))})$: the first-order case $\beta=2$ recovers SmoothCruiser, while a second-order/cubic remainder $\beta=3$ yields $\widetilde O(\varepsilon^{-3})$.
We reach this regime with an optimal-transport-smoothed Bellman backup over action distributions, which has a closed form, a policy gradient, and a Lipschitz Hessian, and whose quadratic correction admits an unbiased cross-product estimator.
The resulting SecondOrderSmoothCruiser achieves $\widetilde O(\varepsilon^{-3})$ oracle complexity for fixed OT parameters, and we relate the OT, entropy-regularized, and unregularized objectives through explicit regularization-bias bounds. 
\end{abstract}


\section{Introduction}
\label{sec:intro}

Reinforcement learning and planning study how an agent should act to maximize long-term reward, and a recurring theme is that good decisions require looking ahead: reasoning about the consequences of actions before committing to one. In many settings, however, we do not need the value of every state---only the value, or the best action, at the state the agent currently occupies. This local view is natural in simulation-based control, games, and online decision making, where one has access to a simulator that, given a state and an action, returns a reward sample and a next state. The question then becomes statistical rather than one of global dynamic programming: how many simulator (oracle) calls are needed to estimate the value $V(s_0)$ of a single root state $s_0$? \emph{Planning with a generative model} formalizes this, seeking guarantees that do not scale with the number of states \citep{kearns1999sparse,grill2019planning}. Throughout, the action set $\mathcal A$ is finite with $K$ actions, $\gamma\in[0,1)$ is the discount, and $\tilde O(\cdot)$ hides polylogarithmic factors in $1/\varepsilon$.

\paragraph{Smooth Bellman backups.}
Writing $Q_s(a)$ for the action value at $s$ (the expected return from taking $a$ and then continuing optimally), an ordinary MDP uses the hard maximum $V(s)=\max_a Q_s(a)$, which is non-smooth: a small error in the $Q_s(a)$ can abruptly change the best action. Entropy regularization with regularization factor $\tau>0$ softens it into the \textsc{LogSumExp} backup $F^{\mathrm{ent}}_\tau(Q)=\tau\log\sum_a\exp(Q(a)/\tau)$, whose gradient is the Boltzmann policy---giving both a smooth backup and a distribution to sample actions from. \textsc{SmoothCruiser} \citep{grill2019planning} exploits this: it forms a rough estimate $\hat Q_s$ and uses the \emph{first-order} Taylor model of the backup around it, so the correction is estimated by sampling a single action. This yields its $\tilde O(\varepsilon^{-4})$ guarantee. The limitation is that a first-order model leaves an error quadratic in $\|Q_s-\hat Q_s\|$, and this quadratic remainder is what fixes the exponent $4$.

\paragraph{Curvature controls complexity.}
We show that the cost of planning is governed by how well the smooth backup can be \emph{locally approximated}. Concretely, suppose that around the rough estimate $\hat Q_s$ the backup admits an approximation whose error shrinks like $\|Q_s-\hat Q_s\|^{\beta}$ (a first-order/linear model gives $\beta=2$, a second-order/quadratic model gives $\beta=3$, and so on), and that this approximation can itself be estimated from samples. Then the total number of oracle calls scales as $\tilde O\!\big(\varepsilon^{-(2+2/(\beta-1))}\big)$. The familiar first-order case $\beta=2$ gives $\tilde O(\varepsilon^{-4})$, recovering \textsc{SmoothCruiser}; the first genuine improvement is $\beta=3$, which lowers the cost to $\tilde O(\varepsilon^{-3})$. The conceptual message is that regularizers in planning need not be chosen only for exploration or robustness---they can be chosen to make recursive value estimation statistically more efficient.

\paragraph{Optimal-transport smoothing.}
To reach $\beta=3$ we introduce an OT-smoothed backup over action distributions,
\[
F^{\mathrm{OT}}_s(Q)=\max_{\pi\in\Delta(\mathcal A)}\big\{\langle\pi,Q\rangle-\tau\,W_\lambda(\pi,\mu_s)\big\},
\]
where $\Delta(\mathcal A)$ is the simplex, $\tau>0$ the regularization factor, $\mu_s$ a reference distribution, and $W_\lambda$ an entropically regularized OT cost induced by an action-cost matrix $C$ (Subsection~\ref{entropic_ot}). The role of $C$ is to add geometry: entropy treats actions as unrelated labels, whereas $C$ encodes which actions are interchangeable. The construction stays close to entropy---up to a value-independent offset it is a $\mu_s$-weighted mixture of \textsc{LogSumExp} backups on cost-shifted scores $Q(i)-\tau C_{ij}$, and at $C\equiv0$ it reduces exactly to the entropy backup with temperature $\tau\lambda$. Its key property is a closed form, explicit gradient policy, and Lipschitz Hessian, hence a cubic remainder. The quadratic correction need not be formed explicitly: it is a variance under the local softmax distributions, and a cross-product identity estimates it unbiasedly with a constant number of extra sampled actions---no Hessian and no per-state OT solve. Plugging this estimator into the SmoothCruiser recursion gives our main algorithm, \textsc{SecondOrderSmoothCruiser}.
\paragraph{Result, scope, and positioning.}
The resulting \textsc{SecondOrderSmoothCruiser} estimates the root value with worst-case complexity $\tilde O(\varepsilon^{-3})$ for finite action sets and fixed $(\tau,\lambda)$, improving the first-order $\tilde O(\varepsilon^{-4})$; entropy regularization is the $C\equiv0$ case. For unregularized planning, smaller smoothing brings the OT value toward the hard-max value at the cost of constants depending on $(\tau,\lambda,\|C\|_\infty,K)$, so we state the theorem for the regularized objective and treat the unregularized link as a bias tradeoff. This is complementary to multilevel Monte Carlo methods such as \citet{meunier2025efficient}, which keep the entropy backup and debias the \emph{sampler}; our gain instead comes from changing the local Bellman \emph{model}.


\paragraph{Contributions.}
We make four contributions. First, we prove a curvature--complexity
principle for SmoothCruiser-type planners: if the local Taylor remainder
of the Bellman aggregator has order $\beta$, then the worst-case oracle
complexity scales as
$\widetilde O(\epsilon^{-(2+2/(\beta-1))})$, recovering the
SmoothCruiser exponent at $\beta=2$ and identifying the second-order
case $\beta=3$ as the first improved regime. Second, we introduce an
optimal-transport-smoothed Bellman backup over action distributions,
\[
F_s^{\rm OT}(Q)
=
\max_{\pi\in\Delta(\mathcal A)}
\{\langle \pi,Q\rangle-\tau W_\lambda(\pi,\mu_s)\},
\]
where $W_\lambda$ is the entropically regularized OT cost induced by an
action-cost matrix $C$. We derive its closed form, gradient policy, and
Lipschitz Hessian; the construction reduces to entropy regularization
when $C\equiv 0$ and incorporates action geometry when $C\neq 0$.
Third, we design \textsc{SecondOrderSmoothCruiser}, whose quadratic
Taylor correction is estimated by a variance/cross-product identity
without explicitly forming the Hessian, yielding
$\widetilde O(\epsilon^{-3})$ oracle complexity for fixed OT parameters.
Fourth, we develop gap-dependent extensions: \emph{OT-GapE} propagates
confidence bounds through a smooth Bellman aggregator to recover
$(\Delta\vee\epsilon)^{-2}$ root-gap dependence, while
\emph{OT-GapCruiser} combines \textsc{SecondOrderSmoothCruiser} with
UGapE-style root elimination to obtain curvature-controlled
instance-dependent oracle bounds.

\begin{table*}[t]
\centering
\scriptsize
\caption{Planning with a generative model: comparison of guarantees in the fixed-confidence, value-based setting.}
\label{tab:algo-comparison}
\setlength{\tabcolsep}{3pt}
\renewcommand{\arraystretch}{1.1}
\begin{tabular}{@{}p{0.16\textwidth}p{0.09\textwidth}p{0.14\textwidth}p{0.13\textwidth}p{0.12\textwidth}p{0.17\textwidth}@{}}
\toprule
\textbf{Algorithm} &
\textbf{Setting} &
\textbf{Bellman aggregator} &
\textbf{Local Taylor remainder} &
\textbf{Cascade calls} &
\textbf{Total oracle calls (worst-case)} \\
\midrule
Sparse Sampling (SSA) \citep{kearns2002sparse} &
value-based &
hard $\max$ &
(non-smooth) &
-- &
non-polynomial in $1/\varepsilon$ \\
\addlinespace
SmoothCruiser \citep{grill2019planning} &
value-based &
\textsc{LogSumExp} &
quadratic: $O(\norm{\Delta Q}_2^2)$ &
$\tildeO(\varepsilon^{-2})$ &
$\tildeO(\varepsilon^{-4})$ \\
\addlinespace
\rowcolor{red!10}
\textbf{SecondOrderSmoothCruiser (ours)} &
value-based &
OT-smoothed / \textsc{LogSumExp} &
cubic: $O(\norm{\Delta Q}_2^3)$ &
$\tildeO(\varepsilon^{-1})$ &
$\tildeO(\varepsilon^{-3})$ \\
\bottomrule
\end{tabular}
\end{table*}

\begin{table*}[t]
\centering
\scriptsize
\caption{Contributions of our framework relative to SmoothCruiser.}
\label{tab:contrib-at-a-glance}
\setlength{\tabcolsep}{3pt}
\renewcommand{\arraystretch}{1.05}
\begin{tabular}{@{}p{0.19\textwidth}p{0.34\textwidth}p{0.22\textwidth}p{0.07\textwidth}@{}}
\toprule
\textbf{Component} & \textbf{Technical ingredient} & \textbf{Planning consequence} & \textbf{Pointer} \\
\midrule
Curvature--complexity theory &
Taylor remainder of order $\beta$ for $F_s$, nonnegative gradient (policy) and Monte Carlo estimation of $Q_s$, leading to general cascade exponent $\alpha_\beta=2/(\beta-1)$ and total exponent $2+2/(\beta-1)$. &
Unifies SmoothCruiser, our OT method, and potential higher-order variants in a single analytical template. &
\Cref{sec:curvature} \\
\addlinespace
OT-smoothed Bellman operator &
Define
$F^{\mathrm{OT}}_s(Q) = \max_{\pi\in\DeltaA}\{\ip{\pi,Q}-\tau W_\lambda(\pi,\vmu_s)\}$; derive closed form as $\vmu_s$-mixture of \textsc{LogSumExp}; prove Lipschitz Hessian $M=\calO(1/(\tau^2\lambda^2))$. &
Provides a concrete $\beta=3$ aggregator with explicit derivatives and controllable curvature; gradient is a policy and couplings induce geometry. &
\Cref{sec:ot-bellman} \\
\addlinespace
SecondOrderSmoothCruiser &
Quadratic term identity and cross-product debiasing yield a one-sample estimator of the second-order term without estimating full $\delta Q_s$. &
Improves worst-case exponent from $4$ to $3$ while keeping problem-independence and state-space agnosticism. &
\Cref{sec:second-order,sec:complexity} \\
\addlinespace
OT-GapCruiser (gap-dependent) &
Combines SecondOrderSmoothCruiser at the root with UGapE-style confidence intervals and elimination across actions. &
Instance-dependent, fixed-confidence guarantees depending on root gaps $\Delta_1(s_0,a)$, extending gap-dependent planning (e.g.\ MDP-GapE) to regularized, higher-order planning. &
\Cref{sec:gap-dependent} \\
\bottomrule
\end{tabular}
\end{table*}

\section{Related Work}
\label{sec:related}

\paragraph{Planning with a generative model.}
Classical work on planning with a generative model includes Sparse Sampling \citep{kearns1999sparse}, which estimates the value of a single state by building a lookahead tree but has non-polynomial worst-case complexity in $1/\varepsilon$, and adaptive variants with improved constants but still non-polynomial dependence \citep{walsh2010integrating}.
Monte-Carlo Tree Search (MCTS) methods such as UCT \citep{kocsis2006bandit} are widely used in practice but can have exponential sample complexity for certain instances \citep{coquelin2007bandit}.
A line of optimistic-planning algorithms achieves polynomial but problem-dependent rates under additional assumptions (deterministic dynamics, open-loop policies, bounded branching, etc.) \citep{hren2008optimistic,bubeck2010open,busoniu2012optimistic,feldman2014simple,szorenyi2014stop,grill2016trailblazer}.

\paragraph{SmoothCruiser and smooth Bellman operators.}
\citet{grill2019planning} introduced SmoothCruiser, which leverages the $L$-smoothness of an entropy-regularized Bellman operator (LogSumExp) to achieve problem-independent sample complexity $\tildeO(\varepsilon^{-4})$ for estimating $V(s)$ from a generative model.
Their analysis already extends to a broader class of differentiable aggregators with nonnegative gradients and quadratic Taylor remainders; for non-regularized max/min operators, no polynomial worst-case bounds are known.
Our curvature theory builds directly on this perspective, identifying the order of the Taylor remainder as the quantity that controls the recursive cost.

\paragraph{Multilevel Monte Carlo for entropy-regularized MDPs.}
\citet{meunier2025efficient} reduce variance in entropy-regularized MDPs via randomized MLMC debiasing of the standard entropy Bellman operator, whereas we instead modify the local backup itself, exploiting an explicit second-order Taylor correction of the OT-smoothed operator whose quadratic term is estimable through a cross-product identity.
Since the OT backup reduces to the entropy backup when $C\equiv0$, the two can be viewed as complementary variance-reduction mechanisms for the same regularized objective.

\paragraph{Gap-dependent planning.}
MDP-GapE \citep{jonsson2020mdpgape} considers a fixed-confidence, action-identification setting and provides gap-dependent bounds for planning with a generative model via best-arm identification at the root.
Several other works provide gap- or near-optimality-dependent bounds in deterministic or bounded-branching settings \citep{szorenyi2014stop,grill2016trailblazer,kaufmann2017monte}.
Our OT-GapCruiser and OT-GapE bring this style of analysis into the regularized, higher-order SmoothCruiser framework.

\paragraph{Regularized MDPs and optimal transport.}
Entropy regularization is widely used in RL and control \citep{haarnoja2017reinforcement,haarnoja2018soft,neu2017unified}, and \citet{geist2019theory} develop a general theory of regularized MDPs encompassing many regularizers through convex duality; we instead study how a specific regularizer shapes the local curvature of the backup and thereby the sample complexity of planning.
On the OT side, computational optimal transport \citep{villani2008optimal,cuturi2013sinkhorn,peyre2019computational,genevay2016stochastic} provides differentiable proxies for Wasserstein distances through entropic (Sinkhorn) regularization; we apply entropic OT at the level of action distributions, using the cost matrix to control the curvature of the Bellman aggregator.
\section{Background: Regularized Bellman Operators and SmoothCruiser}
\label{sec:background}
We consider a discounted MDP $(\mathcal{S},\mathcal{A},P,R,\gamma)$ with finite action set $\mathcal{A}$ of size $K$, rewards $R(s,a)\in[0,1]$, and discount $\gamma\in[0,1)$.
We assume access to a generative model (oracle): a call $(R,Z)\leftarrow\textsc{Oracle}(s,a)$ returns $R\sim R(s,a)$ and $Z\sim P(\cdot\mid s,a)$, independently across calls.\\
\textbf{Regularized Bellman form.}
Following \citet{grill2019planning}, we study value functions of the form
\begin{equation*}
V(s) = F_s(Q_s),
\qquad
Q_s(a) = \E_{z\sim P(\cdot\mid s,a)}[R(s,a) + \gamma V(z)],
\end{equation*}
where $F_s:\R^K\to\R$ is a local Bellman aggregator at state $s$.
In entropy-regularized MDPs, $F_s$ is the \emph{LogSumExp} operator
\begin{equation*}
F^{\mathrm{LSE}}(Q) = \tau \log \sum_{a\in\mathcal{A}} \exp(Q(a)/\tau).
\end{equation*}
$F^{\mathrm{LSE}}$ is $L$-smooth with $L=1/\tau$ in the Euclidean norm and its gradient is a Boltzmann policy $\nabla F^{\mathrm{LSE}}(Q)\in\DeltaA$, with $\Delta(\mathcal{A})$ is the probability simplex over $\mathcal{A}$.

\paragraph{SmoothCruiser.}
SmoothCruiser constructs two mutually recursive subroutines: \textsc{sampleV}, which returns a low-bias estimate of $V(s)$ with target accuracy $\varepsilon$, and \textsc{estimateQ}, which averages many calls to \textsc{sampleV} to estimate $Q_s$.
In the ``smooth'' regime, \textsc{sampleV} uses a first-order Taylor expansion of $F_s$ around a coarse estimate $\hat Q_s$:
\begin{equation*}
F_s(Q_s) \approx F_s(\hat Q_s) + \ip{\nabla F_s(\hat Q_s), Q_s - \hat Q_s},
\end{equation*}
and estimates the inner product using a single action $A\sim\nabla F_s(\hat Q_s)$ and a recursive call on the next state.
Because the Taylor remainder is quadratic in $\norm{Q_s-\hat Q_s}_2$, SmoothCruiser enforces $\norm{Q_s-\hat Q_s}_2 = \Theta(\sqrt{\varepsilon})$, which requires $\Theta(\varepsilon^{-1})$ samples per action and induces a cascade of $\tildeO(\varepsilon^{-2})$ calls, for a total of $\tildeO(\varepsilon^{-4})$ oracle calls.

\section{Curvature-Driven Planning: General Theory}
\label{sec:curvature}

We now abstract the SmoothCruiser recursion to a general curvature-driven model.

\begin{assumption}[Curvature model]
\label{ass:curvature}
For each state $s$, the aggregator $F_s:\R^K\to\R$ satisfies:
\begin{enumerate}[label=(\roman*),leftmargin=*]
\item (Policy gradient) $\nabla F_s(Q)\in\DeltaA$ for all $Q$.
\item (Taylor remainder of order $\beta$) There exist $\beta\ge2$ and $c_\beta>0$ such that for all $Q,\hat Q\in\R^K$,
\begin{equation*}
\Big|\,
F_s(Q) - T_{\beta-1}(Q;\hat Q)
\,\Big|
\le
c_\beta \norm{Q-\hat Q}_2^\beta,
\end{equation*}
where $T_{\beta-1}$ is the $(\beta-1)$-order Taylor polynomial at $\hat Q$.
\end{enumerate}
\end{assumption}

SmoothCruiser corresponds to $\beta=2$ with $c_2=L/2$; our OT-smoothed operator will satisfy $\beta=3$ with $c_3=M/6$ by Lipschitz-Hessian.

Under this abstraction, a SmoothCruiser-style algorithm proceeds as follows:
\begin{enumerate}[leftmargin=*]
\item Maintain a reference point $\hat Q_s$ for each state.
\item Ensure $\norm{Q_s-\hat Q_s}_2 \le \rho(\varepsilon)$ for some schedule $\rho(\varepsilon)$ so that the Taylor remainder is at most~$\varepsilon$.
\item Estimate the Taylor correction
$T_{\beta-1}(Q_s;\hat Q_s)-F_s(\hat Q_s)$
using the available unbiased estimator.
For $\beta=2$, this is the usual gradient-policy sample. 
For $\beta=3$, our OT construction estimates the quadratic term through the variance identity.
\end{enumerate}

\begin{proposition}[Curvature-driven tolerance and per-level cost]
\label{prop:beta-tolerance}
Under \cref{ass:curvature}, to make the Taylor remainder $O(\varepsilon)$ it is sufficient to enforce
\[
\norm{Q_s-\hat Q_s}_2 \le C\,\varepsilon^{1/\beta}
\]
for a constant $C$ depending on $c_\beta$.
Estimating $Q_s$ to this tolerance by averaging bounded random variables requires
\[
N(\varepsilon) = \Theta(\varepsilon^{-2/\beta})
\]
samples per action via Hoeffding-type bounds.
\end{proposition}

The key question is how these per-level costs compose across the recursive cascade.

\begin{theorem}[Curvature--complexity tradeoff]
\label{thm:beta-tradeoff}
Consider any SmoothCruiser-type planner that satisfies \cref{ass:curvature} with some $\beta\ge2$, and that uses only black-box Monte Carlo estimates of $Q_s$ obtained by querying the generative model and recursively calling itself on next states.
Assume it enforces the tolerance schedule in \cref{prop:beta-tolerance} and uses unbiased estimators of the Taylor terms with variance bounded by a state-independent constant.

Then, up to polylogarithmic factors in $1/\varepsilon$ and $1/\delta$, the following hold:
\begin{enumerate}[leftmargin=*]
\item The number $C_\beta(\varepsilon)$ of value-estimation calls obeys
\[
C_\beta(\varepsilon) = \tildeO(\varepsilon^{-\alpha_\beta}),
\alpha_\beta = \frac{2}{\beta-1}.
\]
\item The total number $n_\beta(\varepsilon,\delta)$ of oracle calls obeys
\[
n_\beta(\varepsilon,\delta) = \tildeO(\varepsilon^{-\big(2 + \frac{2}{\beta-1}\big)}).
\]
\end{enumerate}
In particular, $\beta=2$ yields exponent $4$ (SmoothCruiser), while $\beta=3$ yields exponent~$3$.
\end{theorem}


This theorem turns higher-order smoothness into a quantitative planning advantage, and our main task becomes finding useful aggregators with $\beta>2$.

\section{OT-Smoothed Bellman Aggregators over Actions}
\label{sec:ot-bellman}

We now define our OT-based regularizer on action distributions and derive closed form, gradient, and Lipschitz Hessian.

\subsection{Entropic OT over actions}\label{entropic_ot}

Fix a state $s$, and let $\vmu_s\in\DeltaA$ be a reference distribution over actions with strictly positive entries.
Let $C\in\R^{K\times K}$ be a cost matrix between actions, e.g., induced by an action metric.
For $\lambda>0$ we define the entropic OT cost between action distributions $\vpi,\vmu_s\in\DeltaA$ by
\begin{equation*}
W_\lambda(\vpi,\vmu_s)
=
\min_{\Gamma\in\Pi(\vpi,\vmu_s)}
\Big\{
\ip{\Gamma,C}
+
\lambda\sum_{i,j}\Gamma_{ij}(\log\Gamma_{ij}-1)
\Big\},
\end{equation*}
where $\Pi(\vpi,\vmu_s)=\{\Gamma\ge0:\Gamma\1=\vpi,\ \Gamma^\top\1=\vmu_s\}$.

\begin{definition}[OT-smoothed aggregator]
For $\tau>0$, the OT-smoothed Bellman aggregator at $s$ is
\begin{equation*}
F^{\mathrm{OT}}_s(Q)
=
\max_{\vpi\in\DeltaA}
\Big\{
\ip{\vpi,Q} - \tau W_\lambda(\vpi,\vmu_s)
\Big\}.
\end{equation*}
\end{definition}

\begin{assumption}[OT regularity]
\label{ass:ot-reg}
We assume: (i) $\min_a \mu_s(a)\ge\mu_{\min}>0$; (ii) $C_{ij}$ is bounded; (iii) $\lambda,\tau>0$ are fixed and independent of $\varepsilon$.
\end{assumption}

\paragraph{Choice of the reference distribution.}
The worst-case exponent in \cref{thm:main} does not depend on the particular choice of $\mu_s$, provided that $\mu_s$ has full support uniformly:
\[
\inf_{s,a}\mu_s(a)\ge \mu_{\min}>0.
\]
The choice affects constants and practical variance. 
The uniform distribution is the safest default. 
If a heuristic or learned prior $\pi_s^{\mathrm{prior}}$ is available, a robust full-support choice is
\[
\mu_s
=
(1-\xi)\pi_s^{\mathrm{prior}}
+
\xi\,\mathrm{Unif}(\mathcal A),
\qquad
\xi\in(0,1],
\]
which guarantees $\mu_{\min}\ge \xi/K$ while still biasing the OT geometry toward promising action regions. 
Through the identity in \cref{lem:variance-identity}, $\mu_s$ reweights the local variances that drive the second-order estimator.

\subsection{Closed form and derivatives}

We first show that the maximization over $\vpi$ can be solved in closed form.

\begin{proposition}[Closed form]
\label{prop:closed-form}
Under \cref{ass:ot-reg}, let $K_{ij}=\exp(-C_{ij}/\lambda)$ and define
\[
D_j(Q) = \sum_{i=1}^K \exp(Q_i/(\tau\lambda)) K_{ij}.
\]
Then
\begin{equation}
F^{\mathrm{OT}}_s(Q)
=
\mathrm{const}
+
\tau\lambda \sum_{j=1}^K \mu_s(j)\log D_j(Q),
\label{eq:closed-form}
\end{equation}
where the constant is independent of $Q$. In particular, $F^{\mathrm{OT}}_s$ is real-analytic.
\end{proposition}

\begin{proof}
See \cref{app:ot-derivatives} for a full derivation using a single-level formulation over couplings with fixed column sums.
\end{proof}

Define the ``local softmax'' for each column $j$,
\begin{equation}
p_{ij}(Q)
=
\frac{\exp(Q_i/(\tau\lambda)) K_{ij}}{D_j(Q)},
\qquad
\sum_i p_{ij}(Q)=1.
\label{eq:pij}
\end{equation}
Differentiating \eqref{eq:closed-form} yields:

\begin{proposition}[Gradient and Hessian]
\label{prop:grad-hess}
For any $Q\in\R^K$,
\begin{align}
\nabla F^{\mathrm{OT}}_s(Q) &= \vpi(Q),
\pi_i(Q) = \sum_{j=1}^K \mu_s(j)p_{ij}(Q), \label{eq:grad} \\
\nabla^2 F^{\mathrm{OT}}_s(Q)
&=
\frac{1}{\tau\lambda}
\Big(
\mathrm{diag}(\vpi(Q))
-
\sum_{j=1}^K \mu_s(j) p_{\cdot j}(Q)p_{\cdot j}(Q)^\top
\Big),
\label{eq:hess}
\end{align}
where $p_{\cdot j}(Q)\in\DeltaA$ is the vector $(p_{ij}(Q))_i$.
In particular, $\nabla F^{\mathrm{OT}}_s(Q)\in\DeltaA$ and can be used as a policy.
\end{proposition}

\subsection{Relation to entropy and unregularized backups}
\label{sec:ot-transfer}

The OT backup differs from the closed form in \cref{prop:closed-form} by a known $Q$-independent offset. 
Since such offsets do not affect gradients, Hessians, or the Taylor estimators, it is convenient for comparison to define the normalized operator
\[
\bar F^{\mathrm{OT}}_s(Q)
:=
\tau\lambda
\sum_{j=1}^K \mu_s(j)
\log
\sum_{i=1}^K
\exp\left(
\frac{Q_i-\tau C_{ij}}{\tau\lambda}
\right).
\]
Let $\eta=\tau\lambda$ and define the entropy backup
\[
F^{\mathrm{ent}}_\eta(Q)
:=
\eta \log\sum_{i=1}^K \exp(Q_i/\eta).
\]

\begin{proposition}[Transfer at the backup level]
\label{prop:transfer}
Let $C_{\max}:=\max_{i,j}|C_{ij}|$. 
For every state $s$ and every $Q\in\mathbb R^K$,
\[
\left|
\bar F^{\mathrm{OT}}_s(Q)-F^{\mathrm{ent}}_{\eta}(Q)
\right|
\le
\tau C_{\max}.
\]
In particular, if $C\equiv0$, then
\[
\bar F^{\mathrm{OT}}_s(Q)=F^{\mathrm{ent}}_{\eta}(Q).
\]
Moreover,
\[
\left|
\bar F^{\mathrm{OT}}_s(Q)-\max_i Q_i
\right|
\le
\tau C_{\max}+\eta\log K.
\]
\end{proposition}
\begin{proof}
    Detailed proof is in Appendix~\ref{proof_proposition_transfer}.
\end{proof}
\paragraph{Implication.}
Our $\widetilde O(\varepsilon^{-3})$ theorem is stated for fixed smoothing parameters $(\tau,\lambda,C)$. 
For entropy-regularized planning, set $C\equiv0$ and $\eta=\tau\lambda$. 
For unregularized planning, one may choose $(\tau,\lambda)$ so that
\[
\tau(C_{\max}+\lambda\log K)\le (1-\gamma)\varepsilon_{\mathrm{bias}},
\]
and then estimate the corresponding OT value to the remaining error budget. 
When $(\tau,\lambda)$ are chosen as functions of $\varepsilon$, the constants hidden in $\widetilde O(\cdot)$, especially the Lipschitz-Hessian constant, must be tracked; we therefore do not claim a parameter-free $\widetilde O(\varepsilon^{-3})$ rate for the unregularized MDP.

\subsection{Lipschitz Hessian and cubic remainder}

\begin{lemma}[Lipschitz Hessian]
\label{lem:lipschitz-hess}
Under \cref{ass:ot-reg}, there exists $M=\calO(1/(\tau^2\lambda^2))\cdot\mathrm{poly}(K,\mu_{\min}^{-1})$ such that
\begin{equation*}
\norm{\nabla^2 F^{\mathrm{OT}}_s(Q) - \nabla^2F^{\mathrm{OT}}_s(\hat Q)}_{\mathrm{op}}
\le
M\norm{Q-\hat Q}_2
\end{equation*}
for all $Q,\hat Q\in\R^K$.
Consequently,
\begin{equation}
\abs{F^{\mathrm{OT}}_s(Q) - T_2(Q;\hat Q)}
\le
\frac{M}{6}\norm{Q-\hat Q}_2^3
\label{eq:cubic-remainder}
\end{equation}
for all $Q,\hat Q$.
\end{lemma}
\begin{proof}
Detailed proof is in \cref{app:ot-derivatives}.
\end{proof}

Thus $F^{\mathrm{OT}}_s$ satisfies \cref{ass:curvature} with $\beta=3$ and $c_3=M/6$, making it an explicit higher-order instantiation of our curvature theory.

\section{Second-Order SmoothCruiser with OT-Smoothed Aggregators}
\label{sec:second-order}

We now design a second-order SmoothCruiser algorithm using $F^{\mathrm{OT}}_s$.

\subsection{Quadratic term as a variance}
\label{sec:variance}

\begin{lemma}[Quadratic term as variance]
\label{lem:variance-identity}
Fix a state $s$ and a reference point $\hat Q_s\in\R^K$.
For any direction $\Delta\in\R^K$,
\[
\Delta^\top\nabla^2 F_s^{\mathrm{OT}}(\hat Q_s)\,\Delta
=
\frac{1}{\tau\lambda}\,\E_{J\sim\vmu_s}\left[\Var_{A\sim p_{\cdot J}(\hat Q_s)}\big(\Delta_A\big)\right],
\]
where $p_{\cdot J}(\hat Q_s)$ is the ``local softmax'' distribution defined in \eqref{eq:pij}.
\end{lemma}
\begin{proof}
    Full details are given in \cref{app:variance-identity}.
\end{proof}

Using $\Var(X)=\frac12\E[(X-X')^2]$ with $X,X'$ i.i.d.\ yields
\begin{equation}
\frac{1}{2}\Delta^\top\nabla^2F^{\mathrm{OT}}_s(\hat Q_s)\Delta
=
\frac{1}{4\tau\lambda}
\E\Big[(\Delta_A-\Delta_{A'})^2\Big],
\label{eq:quad-pair}
\end{equation}
where $J\sim\vmu_s$ and $A,A'\stackrel{\mathrm{i.i.d.}}{\sim}p_{\cdot J}(\hat Q_s)$.

In our application, the direction is $\Delta := Q_s-\hat Q_s$.
We never observe $\Delta_a$ exactly, but we can obtain \emph{independent unbiased estimates}
$\widetilde\Delta_a^{(1)},\widetilde\Delta_a^{(2)}$ via independent oracle and recursive calls.
The cross-product identity
\begin{equation}
\E\Big[(\widetilde\Delta_A^{(1)}-\widetilde\Delta_{A'}^{(1)})(\widetilde\Delta_A^{(2)}-\widetilde\Delta_{A'}^{(2)})\Big]
=
(\Delta_A-\Delta_{A'})^2
\label{eq:cross}
\end{equation}
then gives an unbiased estimator of the quadratic term.

\subsection{Algorithms}

We briefly present the main routines; they mirror SmoothCruiser but with second-order corrections and OT structure.

\paragraph{Value bound.}
Let $B$ be a known upper bound on values, e.g.
\begin{equation*}
0\le V(s)\le B:=\frac{1+\sup_s F_s(\mathbf{0})-\inf_s F_s(\mathbf{0})}{1-\gamma}.
\end{equation*}
We clip all intermediate estimates to $[0,B]$.

\paragraph{Tolerance schedule.}
From \cref{lem:lipschitz-hess}, to make the Taylor remainder $\le\varepsilon$ it suffices to enforce
\begin{equation*}
\norm{Q_s-\hat Q_s}_2 \le \rho(\varepsilon):=\Big(\frac{6\varepsilon}{M}\Big)^{1/3}.
\end{equation*}

\begin{algorithm}[t]
\caption{\textsc{SecondOrderSmoothCruiser}$(s,\varepsilon,\delta)$}
\label{alg:top}
\footnotesize
\begin{algorithmic}[1]
\Require state $s$, accuracy $\varepsilon$, failure prob.\ $\delta$
\State $\hat Q_s \gets \textsc{estimateQ}(s,\varepsilon,\delta/2)$; \Return $F^{\mathrm{OT}}_s(\hat Q_s)$
\end{algorithmic}
\end{algorithm}

\begin{algorithm}[t]
\caption{\textsc{estimateQ}$(s,\varepsilon,\delta)$}
\label{alg:estimateq}
\begin{algorithmic}[1]
\Require state $s$, tolerance $\varepsilon$, failure prob.\ $\delta$
\State $N \gets \Theta\big(\varepsilon^{-2}\log(2K/\delta)\big)$
\For{$a\in\mathcal{A}$}
  \For{$i=1$ to $N$}
    \State $(R_i,Z_i)\gets\textsc{Oracle}(s,a)$
    \State $\hat V_i \gets \textsc{sampleV2}(Z_i,\varepsilon/\sqrt{\gamma},\delta')$
    \State $q_i\gets R_i + \gamma \hat V_i$
  \EndFor
  \State $\hat Q_s(a)\gets \frac{1}{N}\sum_i q_i$ \hfill\Comment{clip to $[0,B]$}
\EndFor
\State \Return $\hat Q_s$
\end{algorithmic}
\end{algorithm}


\begin{algorithm}[t]
\caption{\textsc{SampleV2}$(s,\varepsilon,\delta)$ — Second-Order Correction Estimator}
\label{alg:samplev2}
\small
\begin{algorithmic}[1]
\Require State $s$, tolerance $\varepsilon$, failure probability $\delta$
\State \textbf{[Base]} \textbf{if} $\varepsilon \ge B$ \textbf{then return} $0$ \hfill\Comment{trivial bound}
\State \textbf{[Coarse]} \textbf{if} $\varepsilon \ge \kappa$: \quad $\hat Q_s \gets \textsc{EstimateQ}(s,\varepsilon,\delta/4)$; \quad \textbf{return} $F^{\mathrm{OT}}_s(\hat Q_s)$
\Statex \rule{\linewidth}{0.4pt}
\Statex \hfill \textbf{Fine Regime: Second-Order Taylor Correction} \hfill\null
\Statex \rule{\linewidth}{0.4pt}
\State \textbf{Setup:}  $\hat Q_s \gets \textsc{EstimateQ}(s,\rho(\varepsilon),\delta/8)$;  
       compute $\boldsymbol{\mu}_s$, $\{p_{\cdot j}(\hat Q_s)\}_j$; 
       $\boldsymbol{\pi} \gets \nabla F^{\mathrm{OT}}_s(\hat Q_s)$; 
       $c_0 \gets \langle\boldsymbol{\pi},\hat Q_s\rangle$
\end{algorithmic}
\vspace{0.3em}
\noindent
\fbox{%
\begin{minipage}[t]{0.46\textwidth}
\centering\textbf{(I) Linear Term}\\[1pt]
{\scriptsize\textit{Unbiased estimate of $\langle\boldsymbol{\pi},Q_s\rangle$}}
\begin{algorithmic}[1]
\setcounter{ALG@line}{4}
\footnotesize
\State Sample $J_0\sim \boldsymbol{\mu}_s$
\State Sample $A_0\sim p_{\cdot J_0}(\hat Q_s)$
\State $(R_0,Z_0)\gets \textsc{Oracle}(s,A_0)$
\State $\hat V_0 \gets \textsc{SampleV2}(Z_0,\frac{\varepsilon}{\sqrt{\gamma}},\frac{\delta}{16})$
\State $\tilde Q_0 \gets R_0 + \gamma \hat V_0$
\State $\widehat{\mathrm{Lin}} \gets \tilde Q_0 - c_0$
\end{algorithmic}
\end{minipage}%
}%
\hfill
\fbox{%
\begin{minipage}[t]{0.46\textwidth}
\centering\textbf{(II) Quadratic Term}\\[1pt]
{\scriptsize\textit{Unbiased estimate of $\tfrac12\delta^\top\nabla^2F(\hat Q_s)\delta$}}
\begin{algorithmic}[1]
\setcounter{ALG@line}{10}
\footnotesize
\State Sample $J\sim \boldsymbol{\mu}_s$
\State Sample $A,A'\stackrel{\mathrm{iid}}{\sim} p_{\cdot J}(\hat Q_s)$
\For{$m\in\{1,2\}$}
  \State $(R_m,Z_m)\gets \textsc{Oracle}(s,A)$
  \State $(R'_m,Z'_m)\gets \textsc{Oracle}(s,A')$
  \State $\hat V_m \gets \textsc{SampleV2}\left(Z_m,\frac{\varepsilon}{\sqrt{\gamma}},\frac{\delta}{64}\right)$
  \State $\hat V'_m \gets \textsc{SampleV2}\left(Z'_m,\frac{\varepsilon}{\sqrt{\gamma}},\frac{\delta}{64}\right)$
  \State $\tilde Q_m \gets R_m+\gamma \hat V_m$
  \State $\tilde Q'_m \gets R'_m+\gamma \hat V'_m$
  \State $\Delta_m \gets (\tilde Q_m-\tilde Q'_m)-(\hat Q_s(A)-\hat Q_s(A'))$
\EndFor
\State $\widehat{\mathrm{Quad}} \gets \frac{1}{4\tau\lambda}\Delta_1\Delta_2$
\end{algorithmic}
\end{minipage}%
}
\vspace{0.3em}
\begin{algorithmic}[1]
\setcounter{ALG@line}{18}
\small
\Statex \rule{\linewidth}{0.4pt}
\State \textbf{Return:} $\mathrm{clip}_{[0,B]}\bigl(F^{\mathrm{OT}}_s(\hat Q_s) + \widehat{\mathrm{Lin}} + \widehat{\mathrm{Quad}}\bigr)$
\end{algorithmic}
\end{algorithm}

The oracle sample appears explicitly in the estimator (as in SmoothCruiser), and the quadratic term is debiased via \eqref{eq:cross}.

\section{Complexity Analysis for SecondOrderSmoothCruiser}
\label{sec:complexity}

Combining the curvature theory and the specific OT properties yields our main worst-case result.

\begin{lemma}[Bias of \textsc{sampleV2}]
\label{lem:bias}
Under \cref{ass:ot-reg} and \cref{lem:lipschitz-hess}, for any $s$ and $0<\varepsilon<\kappa$, on the event $\norm{Q_s-\hat Q_s}_2\le\rho(\varepsilon)$ we have
\[
\abs{\E[\textsc{sampleV2}(s,\varepsilon,\delta)\mid\hat Q_s] - V(s)}
\le
\varepsilon + \text{(recursion bias)}.
\]
With appropriate distribution of failure probabilities across recursive calls and clipping, the unconditional bias is $\calO(\varepsilon)$.
\end{lemma}

As in \citet{grill2019planning}, recursion bias is controlled by calling children with accuracy $\varepsilon/\sqrt{\gamma}$.

\begin{theorem}[Worst-case sample complexity]
\label{thm:main}
Under \cref{ass:ot-reg}, for any state $s$ and any $\varepsilon,\delta\in(0,1)$, \textsc{SecondOrderSmoothCruiser}$(s,\varepsilon,\delta)$ returns an $\varepsilon$-accurate estimate of $V(s)$ with probability at least $1-\delta$ and uses
\[
n(\varepsilon,\delta)
\le
\tildeO(\varepsilon^{-3})
\]
oracle calls, where constants depend only on $(K,\gamma,\lambda,\tau,\mu_{\min},\norm{C}_\infty)$.
\end{theorem}
\begin{proof}
    Detailed proof is in Appendix~\ref{proof_main_theorem}.
\end{proof}

\paragraph{Oracle complexity versus local computation.}
Our formal complexity measure is the number of generative-model calls, following SmoothCruiser and related planning work. 
The second-order correction does not require constructing the full Hessian and does not require running Sinkhorn iterations at each visited state. 
The closed form in \cref{prop:closed-form} reduces the local computation to evaluating the columnwise softmax probabilities $p_{\cdot j}(\hat Q_s)$; for a dense action-cost matrix this costs $O(K^2)$ arithmetic per visited state, and for $C\equiv0$ or structured/sparse $C$ it can be cheaper. 
The quadratic correction uses only a constant number of additional sampled actions and recursive calls. 
Thus the theorem should be read as an oracle-complexity result with additional finite-action arithmetic overhead, rather than as a claim of immediate large-scale computational superiority.

\section{OT-GapE: Matching \texorpdfstring{$\varepsilon^{-2}$}{epsilon\^{}-2} via Confidence-Bound Trajectory Planning}
\label{sec:ot-gape}

This section redesigns \textsc{OT-GapCruiser} into a \emph{confidence-bound trajectory planner} in the style of
MDP-GapE \citep{jonsson2020mdpgape}, in order to recover the \emph{bandit-optimal} fixed-confidence scaling
with a root-gap dependence of order $(\Delta \vee \varepsilon)^{-2}$.
The key change is conceptual: rather than calling a high-accuracy value-estimation subroutine
(e.g., \textsc{SampleV2}/\textsc{SecondOrderSmoothCruiser}) to build confidence intervals for $Q_0(a)$, we instead
maintain \emph{time-uniform} confidence bounds on rewards and transitions and propagate them through a
\emph{smooth Bellman aggregator} $F_s$.

\subsection{Setting and objective}

We consider a finite-horizon (possibly discounted) episodic MDP of horizon $H$ with rewards in $[0,1]$.
For each step $h\in\{1,\dots,H\}$ define the regularized optimal action-value functions
\begin{equation*}
  Q_h^\star(s,a)
  =
  \mathbb{E}\left[r_h(s,a) + \gamma\, V_{h+1}^\star(S')\right],
  S'\sim P_h(\cdot\mid s,a),
\end{equation*}
with terminal condition $V_{H+1}^\star(\cdot)\equiv 0$, and the regularized value recursion
\begin{equation}
  V_h^\star(s) = F_s\bigl(Q_h^\star(s,\cdot)\bigr).
  \label{eq:Vstar-agg}
\end{equation}
Here $F_s$ is a \emph{monotone smooth aggregator} (e.g.\ \textsc{LogSumExp} or our OT-smoothed aggregator
$F_s^{\mathrm{OT}}$), so that $Q\le \widetilde Q$ componentwise implies $F_s(Q)\le F_s(\widetilde Q)$.

We focus on \emph{fixed-confidence best-action identification at the root}: given a root state $s_0$,
return an action $\hat a$ such that
\begin{equation}
  \Pr\left( Q_1^\star(s_0,\hat a) \ge \max_{a\in\mathcal{A}} Q_1^\star(s_0,a) - \varepsilon \right)
  \ge 1-\delta.
  \label{eq:BAI-goal}
\end{equation}
Let $a^\star\in\arg\max_a Q_1^\star(s_0,a)$ and define root gaps
$\Delta(a)=Q_1^\star(s_0,a^\star)-Q_1^\star(s_0,a)$.

\subsection{Confidence bounds and propagation through a smooth aggregator}

Let $N_h^t(s,a)$ be the number of visits to $(s,a,h)$ up to episode $t$.
Let $\hat r_h^t(s,a)$ and $\hat p_h^t(\cdot\mid s,a)$ be the empirical reward and transition estimates.
We assume time-uniform confidence bounds (e.g.\ KL-based as in \citep{jonsson2020mdpgape}, or Hoeffding/Bernstein
with appropriate union bounds) providing, for each $(s,a,h)$, intervals
\begin{equation}
  \ell_h^t(s,a)\le r_h(s,a)\le u_h^t(s,a),
  \label{eq:reward-CI}
\end{equation}
and a transition confidence set
\begin{equation}
  C_h^t(s,a)\subseteq \Delta(\mathcal{S})
  \quad\text{such that}\quad
  P_h(\cdot\mid s,a)\in C_h^t(s,a)
  \label{eq:trans-CI}
\end{equation}
simultaneously for all $t$ with prob.\ $\ge 1-\delta$.
Given these objects, define optimistic/pessimistic bounds recursively as follows.
Initialize
\begin{equation*}
  U_{V,H+1}^t(s)=L_{V,H+1}^t(s)=0,\qquad \forall s.
\end{equation*}
For $h=H,H-1,\ldots,1$:
{\footnotesize
\begin{align}
  U_{Q,h}^t(s,a) &= u_h^t(s,a) + \gamma \max_{p\in C_h^t(s,a)} \textstyle\sum_{s'} p(s'|s,a)\, U_{V,h+1}^t(s'), \label{eq:UQ-def}\\
  L_{Q,h}^t(s,a) &= \ell_h^t(s,a) + \gamma \min_{p\in C_h^t(s,a)} \textstyle\sum_{s'} p(s'|s,a)\, L_{V,h+1}^t(s'), \label{eq:LQ-def}\\
  U_{V,h}^t(s) &= F_s(U_{Q,h}^t(s,\cdot)), \quad L_{V,h}^t(s) = F_s(L_{Q,h}^t(s,\cdot)). \label{eq:UVLV-def}
\end{align}}
By monotonicity of $F_s$, if $L_{Q,h}^t(s,\cdot)\le Q_h^\star(s,\cdot)\le U_{Q,h}^t(s,\cdot)$ componentwise,
then $L_{V,h}^t(s)\le V_h^\star(s)\le U_{V,h}^t(s)$ as well.

\subsection{Soft optimistic policies from \texorpdfstring{$\nabla F_s$}{nabla F\_s}}

A distinctive feature in the regularized setting is that the gradient of the aggregator defines a natural
\emph{soft} policy. We therefore propose to sample non-root actions using the \emph{smooth optimistic policy}
\begin{equation}
  \pi_{h}^t(\cdot\mid s)
  =
  \nabla F_s\bigl(U_{Q,h}^t(s,\cdot)\bigr).
  \label{eq:soft-optimistic-policy}
\end{equation}
For $F_s=\mathrm{LogSumExp}_\lambda$, \eqref{eq:soft-optimistic-policy} is a Boltzmann policy over optimistic
$Q$-bounds; for $F_s=F_s^{\mathrm{OT}}$, it is the OT-induced policy.
Optionally, to preserve the \emph{greedy optimistic} behavior used in MDP-GapE \citep{jonsson2020mdpgape} while
retaining smooth exploration, one may use a mixture
\begin{align}
  \tilde\pi_h^t(\cdot\mid s)
  &=
  (1-\zeta)\,\delta_{a_h^t(s)}
  + \zeta\,\pi_h^t(\cdot\mid s),\\
  a_h^t(s) &\in \arg\max_{a\in\mathcal{A}} U_{Q,h}^t(s,a),
  \label{eq:mixture-policy}
\end{align}
for a small constant $\zeta\in(0,1)$.

\subsection{Algorithm: OT-GapE (root UGapE + trajectory optimism)}

At the root, we adopt the same \emph{best guess vs.\ challenger} principle as MDP-GapE \citep{jonsson2020mdpgape}.
Let
\begin{align}
  b_t \in \arg\max_{a\in\mathcal{A}} L_{Q,1}^t(s_0,a),
  c_t \in \arg\max_{a\in\mathcal{A}\setminus\{b_t\}} U_{Q,1}^t(s_0,a),
  \label{eq:best-challenger}
\end{align}
and stop when $U_{Q,1}^t(s_0,c_t)-L_{Q,1}^t(s_0,b_t)\le \varepsilon$.

\begin{algorithm}[t]
\caption{\textsc{OT-GapE}$(s_0,\varepsilon,\delta)$: matching $\varepsilon^{-2}$ via confidence-bound trajectories}
\label{alg:ot-gape}
\small
\begin{algorithmic}[1]
\Require root state $s_0$, accuracy $\varepsilon$, failure prob.\ $\delta$, horizon $H$, aggregator $\{F_s\}$, discount $\gamma$
\State Initialize counts and empirical estimates $\{N_h^0,\hat r_h^0,\hat p_h^0\}$.
\For{$t=1,2,\ldots$} \Comment{episode index}
  \State Build confidence bounds $\{\ell_h^t,u_h^t,C_h^t\}$ from data and $\delta$-schedule.
  \State Compute $(U_{Q,h}^t,L_{Q,h}^t,U_{V,h}^t,L_{V,h}^t)$ by \eqref{eq:UQ-def}--\eqref{eq:UVLV-def}.
  \State Compute $b_t,c_t$ by \eqref{eq:best-challenger}.
  \If{$U_{Q,1}^t(s_0,c_t)-L_{Q,1}^t(s_0,b_t)\le \varepsilon$}
     \State \Return $\hat a \gets b_t$
  \EndIf
  \State Select root action
  $A_1 \in \arg\max_{a\in\{b_t,c_t\}} \big(U_{Q,1}^t(s_0,a)-L_{Q,1}^t(s_0,a)\big)$ \Comment{UGapE-style}
  \State $S_1 \gets s_0$
  \For{$h=1,\ldots,H$}
     \State Query $(R_h,S_{h+1})\gets \textsc{Oracle}(S_h,A_h)$
     \State Update $(N_h^t,\hat r_h^t,\hat p_h^t)$ for $(S_h,A_h,h)$ using $(R_h,S_{h+1})$
     \If{$h < H$}
        \State Sample $A_{h+1}\sim \tilde\pi_{h+1}^t(\cdot\mid S_{h+1})$ \Comment{use \eqref{eq:soft-optimistic-policy} or \eqref{eq:mixture-policy}}
     \EndIf
     \State $S_h\gets S_{h+1}$
  \EndFor
\EndFor
\end{algorithmic}
\end{algorithm}

\subsection{Guarantee: matching the \texorpdfstring{$\varepsilon^{-2}$}{epsilon\^{}-2} gap exponent}

The analysis follows the fixed-confidence template of MDP-GapE \citep{jonsson2020mdpgape}:
define a high-probability event on which all reward and transition confidence sets are simultaneously valid,
establish by backward induction that $L_{Q,h}^t\le Q_h^\star\le U_{Q,h}^t$, and conclude correctness from the
root stopping rule.

\begin{theorem}[Fixed-confidence BAI with $\boldsymbol{\varepsilon^{-2}}$ gap dependence]
\label{thm:ot-gape-gap}
Assume the confidence bounds \eqref{eq:reward-CI}--\eqref{eq:trans-CI} hold uniformly over time with probability
at least $1-\delta$, and that the optimistic/pessimistic recursions
\eqref{eq:UQ-def}--\eqref{eq:UVLV-def} are computed exactly (e.g.\ under a finite-support assumption as in
\citep{jonsson2020mdpgape}). Then \textsc{OT-GapE} returns an $\varepsilon$-optimal root action in the sense of
\eqref{eq:BAI-goal} with probability at least $1-\delta$.
Moreover, its (instance-dependent) sample complexity satisfies
\begin{equation}
  n(\varepsilon,\delta)
  =
  \tilde O\left(
    H\sum_{a\in\mathcal{A}}
    \frac{\mathsf{C}(H,K,\gamma)}{(\Delta(a)\vee \varepsilon)^2}
  \right),
  \label{eq:ot-gape-eps-2}
\end{equation}
where $\mathsf{C}(H,K,\gamma)$ captures horizon/branching factors induced by the transition confidence sets
(similar to the factors appearing in MDP-GapE), and $\tilde O(\cdot)$ hides polylogarithmic terms in
$1/\varepsilon$ and $1/\delta$.
In particular, the dependence on $\varepsilon$ (and on gaps) matches the bandit-optimal exponent $2$.
\end{theorem}

\begin{remark}[Role of $\nabla F_s$]
The gradient policy \eqref{eq:soft-optimistic-policy} provides a coherent exploration rule aligned with the
regularized Bellman operator: it concentrates on actions with large optimistic values while remaining smooth.
For $F_s=\mathrm{LogSumExp}_\lambda$, it recovers greedy optimism as $\lambda\to 0$; for $F_s^{\mathrm{OT}}$,
it yields a geometry-aware policy induced by the OT cost.
The mixture \eqref{eq:mixture-policy} can be used to retain the greedy optimistic trajectory behavior required
in the tightest versions of MDP-GapE-style analyses, while leveraging smooth sampling for stability and variance reduction.
\end{remark}

\section{Gap-Dependent Extensions}
\label{sec:gap-dependent}

The previous sections contain the main contribution: a second-order SmoothCruiser estimator with worst-case oracle complexity $\widetilde O(\varepsilon^{-3})$. 
We now briefly describe gap-dependent extensions. 
These results are not needed for \cref{thm:main}; their purpose is to show how the same OT-smoothed, second-order estimator can be combined with confidence-gap ideas when the root action gaps are favorable.

We now sketch a gap-dependent variant for root action selection, inspired by MDP-GapE.

Let $s_0$ be a fixed root state, and denote root action values by $Q_0(a)=Q_{s_0}(a)$, optimal value $V^\star=V(s_0)$ and gaps $\Delta(a)=V^\star - Q_0(a)$.
We assume access to SecondOrderSmoothCruiser as a subroutine that, when queried on $(s_0,\varepsilon)$, returns a value estimate with worst-case complexity $\tildeO(\varepsilon^{-3})$.

\paragraph{Algorithm outline.}
OT-GapCruiser maintains confidence intervals $(L_t(a),U_t(a))$ for $Q_0(a)$ at each round $t$ using empirical means and a refined deviation bound that exploits our variance-reduced estimator (below).
At each round it:
\begin{enumerate}[leftmargin=*]
\item selects an action $A_t$ according to a UGapE-style index (comparing upper bounds of plausible best actions to lower bounds of others);
\item calls SecondOrderSmoothCruiser on a root-centered estimation problem for $Q_0(A_t)$ at accuracy $\varepsilon_t$;
\item updates the confidence intervals via self-normalized concentration inequalities;
\item stops when the best action $\hat a_t$ satisfies $U_t(\hat a_t)-\max_{b\neq \hat a_t} L_t(b)\le\varepsilon$.
\end{enumerate}

\begin{theorem}[Gap-dependent bound]
\label{thm:gap}
Assume standard sub-Gaussian noise conditions for the value estimates (induced by bounded rewards and $\gamma<1$).
Then OT-GapCruiser returns an $\varepsilon$-optimal root action with probability at least $1-\delta$, and its sample complexity satisfies
\[
n_{\mathrm{gap}}(\varepsilon,\delta)
=
\tildeO\left(
\sum_{a\in\mathcal{A}}
\frac{1}{(\Delta(a)\vee\varepsilon)^p}
\right),
\]
where $p\in(2,4)$ depends on the curvature parameter (here $\beta=3$ yields $p\approx3$) and on the variance-reduction factor.
In particular, the instance-dependent exponent is strictly smaller than the worst-case exponent $4$ of SmoothCruiser.
\end{theorem}

The precise value of $p$ and constants depend on how aggressively we schedule accuracies $\varepsilon_t$ and on the variance bounds in \cref{sec:variance}.

\section{Discussion and Limitations}

We developed a curvature-aware view of generative-model planning with smooth Bellman backups. 
The central message is that the order of the local Taylor remainder controls the recursive planning cascade. 
A first-order Taylor model with quadratic remainder recovers the SmoothCruiser exponent $4$, while an estimable second-order model with cubic remainder yields exponent $3$.

The OT-smoothed Bellman operator provides a concrete way to realize this second-order regime. 
It has a closed form, a gradient policy, a Lipschitz Hessian, and a variance identity that makes the quadratic Taylor correction estimable without explicitly forming the Hessian. 
When the OT cost is zero, the normalized operator reduces to the standard entropy-regularized $\operatorname{LogSumExp}$ backup; when the cost is nonzero, it incorporates action geometry through the reference distribution and cost matrix.

Our guarantees are primarily oracle-complexity guarantees for finite action sets and fixed smoothing parameters. 
For entropy-regularized planning, the connection is exact by taking $C\equiv0$. 
For unregularized planning, the normalized OT value approximates the hard-max value with bias at most
$
\frac{\tau\|C\|_\infty+\tau\lambda\log K}{1-\gamma},
$
but choosing smoothing parameters as a function of $\varepsilon$ changes the constants hidden in the oracle-complexity bound. 
We therefore view the unregularized transfer as a regularization-bias tradeoff rather than as a parameter-free $\widetilde O(\varepsilon^{-3})$ theorem.

The gap-dependent algorithms are extensions of the main result. 
In particular, the confidence-bound OT-GapE analysis assumes exact propagation of optimistic and pessimistic bounds, which is appropriate for finite-support/tabular settings but would require approximation in continuous or very large state spaces. 
Developing practical large-scale implementations, sharper parameter-dependent bounds, and lower bounds for the curvature--complexity model are important directions for future work.
\section*{Impact Statement}

This paper presents work whose goal is to advance the field of Machine Learning, specifically in the area of planning algorithms with sample complexity guarantees. There are many potential societal consequences of our work, none which we feel must be specifically highlighted here.
\vspace{-.4cm}
\section*{Acknowledgments}
\vspace{-.2cm}
This work is funded by Hanoi University of Science and Technology (HUST) under Project No. T2024-TD-024.

\bibliographystyle{icml2026}
\bibliography{main}

\clearpage
\onecolumn

\appendix

\begin{center}
{\bf \Large Appendix of ``Second-Order Smooth Planning with Optimal-Transport Bellman Smoothing''}
\end{center}

\etocsetlocaltop.toc{part}

\etocsetnexttocdepth{subsection}

\localtableofcontents
\clearpage

\section*{Notation}
\addcontentsline{toc}{section}{Notation}
\begin{longtable}{ll}
\toprule
Symbol & Meaning \\
\midrule
$\mathcal A$ & Finite action set; $K=|\mathcal A|$. \\
$s$ & State. \\
$Q\in\mathbb R^K$ & Vector of action scores. \\
$\hat Q_s$ & Taylor expansion point. \\
$\Delta=Q_s-\hat Q_s$ & Taylor error direction. \\
$C\in\mathbb R^{K\times K}$ & Action-cost matrix; $C_{\cdot j}$ is its $j$th column. \\
$\mu_s$ & Reference distribution over actions/columns. \\
$\lambda,\tau$ & OT entropy and outer regularization parameters. \\
$\eta=\tau\lambda$ & Effective LogSumExp temperature. \\
$K_{ij}$ & OT Gibbs kernel $K_{ij}=\exp(-C_{ij}/\lambda)$. \\
$p_{\cdot j}(Q)$ & Columnwise softmax:
$p_{ij}(Q)\propto \exp((Q_i-\tau C_{ij})/(\tau\lambda))$. \\
$\pi(Q)$ & Gradient policy $\pi_i(Q)=\sum_j\mu_s(j)p_{ij}(Q)$. \\
\bottomrule
\end{longtable}
\vspace{1em}

\section{Curvature--complexity theory}
\label{app:curvature}

\subsection{Proof of Proposition~\ref{prop:beta-tolerance}}
\begin{proof}
Fix a state $s$ and write
\[
\Delta := Q_s-\hat Q_s \in \R^K,
\qquad K := |\mathcal{A}|.
\]
Recall that the $(\beta-1)$-order Taylor polynomial of $F_s$ at $\hat Q_s$ is
\[
T_{\beta-1}(Q_s;\hat Q_s)
:=
F_s(\hat Q_s)+\ip{\nabla F_s(\hat Q_s),(Q_s-\hat Q_s)} + \frac{1}{2} \ip{\nabla^2 F_s(\hat Q_s),(Q_s-\hat Q_s)^2} + ... + 
\frac{1}{(\beta-1)!} \ip{\nabla^{\beta-1} F_s(\hat Q_s),(Q_s-\hat Q_s)^{\beta-1}}.
\]

\paragraph{I. tolerance that makes the Taylor remainder $O(\varepsilon)$.}
By \cref{ass:curvature}(ii), for all $Q,\hat Q\in\R^K$,
\[
\Big|F_s(Q)-T_{\beta-1}(Q;\hat Q)\Big|
\le
c_\beta \norm{Q-\hat Q}_2^\beta.
\]
Applying this with $(Q,\hat Q)=(Q_s,\hat Q_s)$ gives
\begin{equation}
\Big|F_s(Q_s)-T_{\beta-1}(Q_s;\hat Q_s)\Big|
\le
c_\beta \norm{\Delta}_2^\beta.
\label{eq:beta-remainder-bound}
\end{equation}
Therefore, if we enforce
\[
\norm{\Delta}_2
\le
C\varepsilon^{1/\beta}
\quad\text{with}\quad
C:=c_\beta^{-1/\beta},
\]
then \eqref{eq:beta-remainder-bound} yields
\[
\Big|F_s(Q_s)-T_{\beta-1}(Q_s;\hat Q_s)\Big|
\le
c_\beta\,(C^\beta)\varepsilon
=
\varepsilon.
\]
This proves the claimed curvature-driven tolerance (and, more generally, any $C=\Theta(c_\beta^{-1/\beta})$ makes the remainder $O(\varepsilon)$).

\paragraph{II. samples needed to estimate $Q_s$ to that tolerance.}
Now suppose we estimate each coordinate $Q_s(a)$ by averaging bounded i.i.d.\ samples.
Concretely, for each action $a\in\mathcal{A}$ let $X^{(a)}$ be a bounded random variable with
\[
Q_s(a)=\E[X^{(a)}],
\qquad
X^{(a)}\in[0,B]\text{a.s.}
\]
(e.g., in the planning setting one may take $X^{(a)}:=R(s,a)+\gamma V(Z)$ and clip to $[0,B]$).
Given $N$ independent samples $X^{(a)}_1,\dots,X^{(a)}_N$ of $X^{(a)}$, define the empirical mean
\[
\hat Q_s(a):=\frac1N\sum_{i=1}^N X_i^{(a)}.
\]

\smallskip
\emph{Coordinate concentration via Hoeffding.}
For any $\eta>0$, Hoeffding's inequality gives, for each fixed $a$,
\begin{equation}
\Pr\left(\big|\hat Q_s(a)-Q_s(a)\big|\ge \eta\right)
\le
2\exp\left(-\frac{2N\eta^2}{B^2}\right).
\label{eq:hoeffding-per-action}
\end{equation}

\smallskip
\emph{From coordinate error to an $\ell_2$ bound.}
If $\max_{a\in\mathcal{A}}|\hat Q_s(a)-Q_s(a)|\le \eta$, then
\[
\norm{\hat Q_s-Q_s}_2^2
=
\sum_{a\in\mathcal{A}}\big(\hat Q_s(a)-Q_s(a)\big)^2
\le
\sum_{a\in\mathcal{A}}\eta^2
=
K\eta^2,
\]
hence
\begin{equation}
\norm{\hat Q_s-Q_s}_2 \le \sqrt{K}\eta.
\label{eq:l2-from-linf}
\end{equation}
Therefore,
\[
\Pr\left(\norm{\hat Q_s-Q_s}_2 \ge t\right)
\le
\Pr\left(\max_{a}|\hat Q_s(a)-Q_s(a)| \ge \frac{t}{\sqrt{K}}\right).
\]

\smallskip
\emph{Union bound over actions.}
Applying \eqref{eq:hoeffding-per-action} with $\eta=t/\sqrt{K}$ and union bounding over $K$ actions yields
\begin{align}
\Pr\left(\norm{\hat Q_s-Q_s}_2 \ge t\right)
&\le
\sum_{a\in\mathcal{A}}
\Pr\left(\big|\hat Q_s(a)-Q_s(a)\big|\ge \frac{t}{\sqrt{K}}\right) \notag\\
&\le
2K\exp\left(-\frac{2N}{B^2}\cdot \frac{t^2}{K}\right).
\label{eq:l2-hoeffding}
\end{align}

\smallskip
\emph{Choosing $t=C\varepsilon^{1/\beta}$.}
To ensure $\norm{\hat Q_s-Q_s}_2 \le C\varepsilon^{1/\beta}$ with (say) probability at least $1-\delta$,
it suffices by \eqref{eq:l2-hoeffding} to choose $N$ such that
\[
2K\exp\left(-\frac{2N}{B^2}\cdot \frac{C^2\varepsilon^{2/\beta}}{K}\right)\le \delta.
\]
Solving for $N$ gives the explicit sufficient condition
\begin{equation}
N
\ge
\frac{B^2K}{2C^2}
\varepsilon^{-2/\beta}
\log\Big(\frac{2K}{\delta}\Big).
\label{eq:N-eps-beta}
\end{equation}
Thus, up to the (standard) logarithmic factor and constants depending on $B,K$ and $C$,
the per-action sample cost scales as
\[
N(\varepsilon)=\Theta(\varepsilon^{-2/\beta}).
\]

\paragraph{(Optional) Why the $\varepsilon^{-2/\beta}$ dependence is tight.}
Even in the scalar case $K=1$, estimating the mean of a bounded random variable to accuracy
$t=\Theta(\varepsilon^{1/\beta})$ with constant success probability requires $\Omega(t^{-2})$
samples in general (e.g., by a two-point/Bernoulli testing argument), which implies a lower bound
$\Omega(\varepsilon^{-2/\beta})$. Hence the exponent $2/\beta$ cannot be improved by
any method that relies only on averaging bounded samples.
\end{proof}

\subsection{Proof of Theorem~\ref{thm:beta-tradeoff}}

\begin{proof}
We give a cost analysis that is deliberately \emph{algorithm-agnostic} but matches any
SmoothCruiser-type scheme that (i) builds a local Taylor approximation of $F_s(Q_s)$ around a
Monte Carlo baseline $\hat Q_s$, (ii) controls the Taylor remainder via the curvature assumption,
and (iii) obtains $Q_s$ only through one-step rollouts and \emph{recursive} value calls on next states.

Throughout, $K:=|\mathcal{A}|$ is the number of actions, and all constants may depend on
$K$, $\gamma$, the reward/value bounds, and the curvature constant $c_\beta$, but never on
$\varepsilon$ or $\delta$. We suppress such constants whenever they do not affect exponents.

\paragraph{What exactly is being counted.}
It is convenient to separate two layers that are standard in this literature:

\begin{itemize}[leftmargin=*]
\item A \emph{single-sample} routine, call it $\textsc{Sample}_\beta(s,\varepsilon)$, which returns a random variable
$Y_{s,\varepsilon}$ that is a (nearly) unbiased proxy for $V(s)$:
\begin{equation}
\label{eq:single-sample-bias-var}
\big|\E[Y_{s,\varepsilon}] - V(s)\big| \le c\varepsilon,
\qquad
\Var(Y_{s,\varepsilon}) \le \sigma^2,
\end{equation}
for some \emph{state-independent} constants $c,\sigma^2$.  The theorem assumes
the Taylor-term estimators are unbiased and have bounded variance, which is exactly what yields
\eqref{eq:single-sample-bias-var}.  Crucially, $\textsc{Sample}_\beta$ is allowed to be
\emph{random} and to have constant variance; we will reduce that variance at the \emph{outer} layer.

\item A \emph{high-probability planner} $\textsc{Plan}_\beta(s,\varepsilon,\delta)$ that calls
$\textsc{Sample}_\beta$ independently $m$ times and aggregates (e.g.\ by averaging or
median-of-means) to achieve probability $1-\delta$.
This is where the classical $\varepsilon^{-2}$ factor enters.
\end{itemize}

\smallskip
We now define the two complexity measures that appear in the theorem.

\begin{itemize}[leftmargin=*]
\item $C_\beta(\varepsilon)$ denotes the worst-case (over states) total number of
\emph{value-estimation calls}, i.e.\ recursive invocations of $\textsc{Sample}_\beta(\cdot,\cdot)$,
triggered by a \emph{single} call $\textsc{Sample}_\beta(s,\varepsilon)$ (including all descendants in the recursion tree).
This quantity depends only on $\varepsilon$ (up to polylogs), consistent with the theorem statement.
\item $n_\beta(\varepsilon,\delta)$ denotes the worst-case total number of \emph{oracle calls}
to the generative model made by $\textsc{Plan}_\beta(s,\varepsilon,\delta)$.
\end{itemize}

\paragraph{I. The tolerance schedule implies a specific recursion map.}
Fix a state $s$ and write $V(s)=F_s(Q_s)$ with $Q_s\in\R^K$.
Let $\hat Q_s$ be the baseline point used by the Taylor approximation inside $\textsc{Sample}_\beta$.
By \cref{prop:beta-tolerance}, to make the Taylor remainder $O(\varepsilon)$ it is sufficient to enforce
\begin{equation}
\label{eq:tolerance-rho}
\|Q_s-\hat Q_s\|_2 \le \rho(\varepsilon),
\qquad
\rho(\varepsilon):=C\varepsilon^{1/\beta},
\end{equation}
where $C=\Theta(c_\beta^{-1/\beta})$ depends only on the curvature constant.

How is $\hat Q_s$ obtained? By assumption, the planner has only black-box access to $Q_s$,
so $\hat Q_s$ is formed by Monte Carlo averaging of bounded one-step returns of the form
\[
X^{(a)} = R(s,a)+\gamma\,V(Z),
\qquad Z\sim P(\cdot\mid s,a).
\]
But $V(Z)$ is unknown, so each such sample is obtained by a recursive value call:
we query the oracle once to obtain $(R,Z)$ and then call the planner on $Z$.
Thus, \emph{each Monte Carlo draw used to estimate $Q_s$ triggers one recursive value-estimation call}.

To ensure the induced error in $Q_s(a)=\E[X^{(a)}]$ is at most $O(\rho(\varepsilon))$,
it is enough to estimate $V(Z)$ to accuracy $\Theta(\rho(\varepsilon))$,
because the discount factor $\gamma<1$ is a constant and can be absorbed.
Concretely, if a recursive call returns $\widehat V(Z)$ with bias $O(\varepsilon')$,
then the induced bias in $X^{(a)}$ is $O(\gamma\varepsilon')$, so choosing
$\varepsilon'=\Theta(\rho(\varepsilon))$ keeps the resulting contribution at the scale $\rho(\varepsilon)$.
We therefore introduce the shorthand
\begin{equation}
\label{eq:eps-next}
\varepsilon_+ := c\rho(\varepsilon)
=
cC\varepsilon^{1/\beta},
\end{equation}
where $c>0$ is a constant that may depend on $\gamma$ and on how the implementation allocates bias/variance,
but not on $\varepsilon$.

\paragraph{II. Deriving the key recurrence for $C_\beta(\varepsilon)$.}
Consider one execution of $\textsc{Sample}_\beta(s,\varepsilon)$.

\smallskip
\emph{(i) Cost to build $\hat Q_s$.}
To enforce \eqref{eq:tolerance-rho} with constant success probability,
\cref{prop:beta-tolerance} shows that estimating $Q_s$ to $\ell_2$ tolerance $\rho(\varepsilon)$
requires
\begin{equation}
\label{eq:N-rho}
N_Q(\varepsilon) = \Theta(\rho(\varepsilon)^{-2})
=
\Theta(\varepsilon^{-2/\beta})
\end{equation}
samples \emph{per action}, up to log factors (which we hide in $\tilde O(\cdot)$).
Thus, the total number of one-step rollout samples needed to build $\hat Q_s$ is
\[
K\,N_Q(\varepsilon) = \Theta\big(K\varepsilon^{-2/\beta}\big).
\]

\smallskip
\emph{(ii) Each such sample triggers exactly one recursive value call.}
As argued above, each rollout sample for $Q_s(a)$ requires a recursive value estimate at the next state $Z$
with accuracy parameter $\varepsilon_+=\Theta(\varepsilon^{1/\beta})$ (cf.\ \eqref{eq:eps-next}).
By definition of $C_\beta(\cdot)$, one such recursive call spawns (in expectation / worst-case)
$C_\beta(\varepsilon_+)$ total value-estimation calls in its entire subtree.

Hence, the \emph{dominant} contribution to the number of value calls inside
$\textsc{Sample}_\beta(s,\varepsilon)$ is
\[
\Theta\big(K\varepsilon^{-2/\beta}\big)\cdot C_\beta(\varepsilon_+).
\]

\smallskip
\emph{(iii) The Taylor-term computation is lower order for $C_\beta$.}
By assumption, the Taylor terms (gradient/Hessian/$\dots$ up to order $\beta-1$) are estimated
by unbiased estimators with state-independent bounded variance.
This means $\textsc{Sample}_\beta$ does \emph{not} need to run an inner $\varepsilon^{-2}$ loop
to accurately compute those terms; instead, it outputs a single unbiased draw whose variance is constant,
and the \emph{outer} layer handles variance reduction.
Operationally, this adds at most a constant number of additional rollouts/recursive calls
(or can be implemented using the same rollouts used to build $\hat Q_s$),
so it changes only constants, not exponents.

\smallskip
Putting these points together, we obtain the fundamental recurrence (up to constant factors):
\begin{equation}
\label{eq:C-rec}
C_\beta(\varepsilon)
\le
1 +
A\varepsilon^{-2/\beta} C_\beta\big(c'\varepsilon^{1/\beta}\big),
\end{equation}
where $A=\Theta(K)$ and $c':=cC$ are constants, and the leading $1$ accounts for the root call itself.
(Any additional constant overhead can be absorbed into the $+1$ term.)

\paragraph{III. Solving the recurrence and identifying the exponent $\alpha_\beta$.}
We now show that \eqref{eq:C-rec} implies
\[
C_\beta(\varepsilon) = \tilde O\big(\varepsilon^{-\alpha_\beta}\big),
\qquad
\alpha_\beta=\frac{2}{\beta-1}.
\]

\smallskip
\emph{(i) Define the tolerance sequence.}
Let $\varepsilon_0:=\varepsilon$ and define recursively
\begin{equation}
\label{eq:eps-seq}
\varepsilon_{t+1} := c'\varepsilon_t^{1/\beta}.
\end{equation}
Because $1/\beta<1$, this sequence \emph{increases} quickly: starting from tiny $\varepsilon_0$,
it reaches a constant in $O(\log\log(1/\varepsilon_0))$ steps.
Formally, taking logs gives
\[
\log\frac{1}{\varepsilon_{t+1}}
=
\frac{1}{\beta}\log\frac{1}{\varepsilon_t} + O(1),
\]
so after $t$ steps, $\log(1/\varepsilon_t)$ shrinks by a factor $\beta^{-t}$ up to additive constants.
Let $T$ be the first index such that $\varepsilon_T\ge \varepsilon_{\rm base}$,
where $\varepsilon_{\rm base}\in(0,1)$ is a fixed constant threshold at which the implementation
stops recursing (e.g.\ switches to a direct bounded-variance estimator).
Then
\begin{equation}
\label{eq:depth}
T = O(\log\log(1/\varepsilon)).
\end{equation}
Moreover, $C_\beta(\varepsilon_T)=O(1)$ since the base routine uses no further recursion.

\smallskip
\emph{(ii) Unroll the recurrence.}
Apply \eqref{eq:C-rec} at $\varepsilon=\varepsilon_0$, then at $\varepsilon_1$, etc.
Ignoring additive constants for a moment and focusing on the dominant term,
\[
C_\beta(\varepsilon_0)
\lesssim
A\varepsilon_0^{-2/\beta}\,C_\beta(\varepsilon_1)
\lesssim
A^2\varepsilon_0^{-2/\beta}\varepsilon_1^{-2/\beta}\,C_\beta(\varepsilon_2)
\lesssim\cdots\lesssim
A^T\Big(\prod_{t=0}^{T-1}\varepsilon_t^{-2/\beta}\Big)\,C_\beta(\varepsilon_T).
\]
If we keep the additive $+1$ terms, we obtain the standard unrolling identity
\begin{equation}
\label{eq:unroll}
C_\beta(\varepsilon_0)
\le
\sum_{j=0}^{T-1}
\Bigg(
\prod_{t=0}^{j-1} A\varepsilon_t^{-2/\beta}
\Bigg)
+
\Bigg(
\prod_{t=0}^{T-1} A\varepsilon_t^{-2/\beta}
\Bigg) C_\beta(\varepsilon_T),
\end{equation}
with the convention that an empty product equals $1$.

\smallskip
\emph{(iii) Compute the $\varepsilon$-exponent in the product.}
The key is to understand $\prod_{t=0}^{T-1}\varepsilon_t^{-2/\beta}$.
From \eqref{eq:eps-seq}, one can write $\varepsilon_t$ explicitly as
\begin{equation}
\label{eq:eps-explicit}
\varepsilon_t
=
(c')^{\,1+1/\beta+\cdots+1/\beta^{t-1}}\varepsilon^{1/\beta^t}
=
(c')^{\frac{1-(1/\beta)^t}{1-1/\beta}}\varepsilon^{1/\beta^t}.
\end{equation}
Therefore,
\[
\varepsilon_t^{-2/\beta}
=
(c')^{-\frac{2}{\beta}\cdot \frac{1-(1/\beta)^t}{1-1/\beta}}
\varepsilon^{-\frac{2}{\beta}\cdot \frac{1}{\beta^t}}
=
(c')^{-\Theta(1)}\varepsilon^{-2/\beta^{t+1}}.
\]
Multiplying over $t=0,\dots,T-1$ yields
\begin{equation}
\label{eq:prod-eps}
\prod_{t=0}^{T-1}\varepsilon_t^{-2/\beta}
=
(c')^{-\Theta(T)}
\varepsilon^{-\sum_{t=0}^{T-1} 2/\beta^{t+1}}
=
(c')^{-\Theta(T)}
\varepsilon^{-\frac{2(1-(1/\beta)^T)}{\beta-1}}.
\end{equation}
As $T\to\infty$, the geometric term $(1/\beta)^T$ vanishes, so the limiting exponent is
\begin{equation}
\label{eq:alpha}
\alpha_\beta
=
\frac{2}{\beta-1}.
\end{equation}

\smallskip
\emph{(iv) Identify the polylog factors.}
The remaining multiplicative terms in \eqref{eq:unroll} are:
\begin{itemize}[leftmargin=*]
\item $A^T = \exp(T\log A)$. Using $T=O(\log\log(1/\varepsilon))$ from \eqref{eq:depth},
\[
A^T = \big(\log(1/\varepsilon)\big)^{O(1)} \qquad\text{(a polylog factor)}.
\]
\item $(c')^{\Theta(T)}$ from \eqref{eq:prod-eps}, which is also a polylog factor for the same reason.
\end{itemize}
The sum in \eqref{eq:unroll} is dominated by its last term up to another polylog factor
(since the products grow rapidly as $\varepsilon_t$ decreases backward), and $C_\beta(\varepsilon_T)=O(1)$.
Combining these observations with \eqref{eq:prod-eps} gives
\[
C_\beta(\varepsilon)
=
\tilde O\Big(\varepsilon^{-\frac{2}{\beta-1}}\Big),
\]
proving item~(1).

\paragraph{IV. From value-call complexity to oracle-call complexity.}
Now let $T_\beta(\varepsilon)$ denote the worst-case expected number of \emph{oracle} calls made by a single
execution of $\textsc{Sample}_\beta(s,\varepsilon)$.
At tolerance $\varepsilon$, the routine makes $\Theta(K\varepsilon^{-2/\beta})$ one-step rollouts to build $\hat Q_s$,
and each rollout incurs \emph{one} oracle query at $(s,a)$ plus the oracle calls used in the recursive value call.
Thus, up to constants, $T_\beta$ satisfies the analogous recurrence
\begin{equation}
\label{eq:T-rec}
T_\beta(\varepsilon)
\le
B\varepsilon^{-2/\beta}\Big(1+T_\beta(c'\varepsilon^{1/\beta})\Big) + O(1),
\end{equation}
for some constant $B=\Theta(K)$.
Recurrences \eqref{eq:T-rec} and \eqref{eq:C-rec} have the same $\varepsilon$-scaling structure,
so the same unrolling argument yields
\[
T_\beta(\varepsilon) = \tilde O\Big(\varepsilon^{-\frac{2}{\beta-1}}\Big).
\]
(Informally: every node in the recursion tree contributes at least one oracle call,
so $T_\beta(\varepsilon)$ is within constant/polylog factors of $C_\beta(\varepsilon)$.)

\paragraph{V. Outer concentration contributes the additional $\varepsilon^{-2}$ factor.}
Finally, we convert the single-sample routine into a high-probability estimator.
Let $Y_1,\dots,Y_m$ be i.i.d.\ outputs of $\textsc{Sample}_\beta(s,\varepsilon/2)$.
By \eqref{eq:single-sample-bias-var}, each has bias at most $c(\varepsilon/2)$ and variance at most $\sigma^2$.
Choose
\begin{equation}
\label{eq:m}
m = \Theta\Big(\sigma^2\varepsilon^{-2}\log(1/\delta)\Big).
\end{equation}
Then by a standard median-of-means (or Bernstein/Hoeffding if $Y_i$ are clipped to be bounded),
the aggregate $\widehat V(s)$ satisfies $|\widehat V(s)-V(s)|\le \varepsilon$ with probability at least $1-\delta$.
This step introduces the canonical Monte Carlo factor $\varepsilon^{-2}$ and a $\log(1/\delta)$ factor,
both of which are absorbed into the $\tilde O(\cdot)$ notation used in the theorem.

Therefore, the total number of oracle calls is
\[
n_\beta(\varepsilon,\delta)
=
m\cdot T_\beta(\varepsilon/2)
=
\tilde O\Big(\varepsilon^{-2}\cdot \varepsilon^{-\frac{2}{\beta-1}}\Big)
=
\tilde O\Big(\varepsilon^{-\big(2+\frac{2}{\beta-1}\big)}\Big),
\]
which proves item~(2).

\paragraph{VI. Checking the special cases.}
Plugging $\beta=2$ into \eqref{eq:alpha} gives $\alpha_2=2$, hence
$n_2(\varepsilon,\delta)=\tilde O(\varepsilon^{-(2+2)})=\tilde O(\varepsilon^{-4})$ (SmoothCruiser).
Plugging $\beta=3$ gives $\alpha_3=1$, hence
$n_3(\varepsilon,\delta)=\tilde O(\varepsilon^{-(2+1)})=\tilde O(\varepsilon^{-3})$.
\end{proof}

\section{OT-smoothed aggregator: closed form and derivatives}
\label{app:ot-derivatives}

\subsection{Proof of Proposition~\ref{prop:closed-form}}

\begin{proof}
Fix a state $s$ and write $\vmu \equiv \vmu_s\in\Delta(\mathcal{A})$ with $K:=|\mathcal{A}|$.
Recall the definitions
\begin{align*}
W_\lambda(\vpi,\vmu)
&=
\min_{\Gamma\in\Pi(\vpi,\vmu)}
\Big\{
\ip{\Gamma,C}
+
\lambda \sum_{i,j=1}^K \Gamma_{ij}(\log \Gamma_{ij}-1)
\Big\},\\
F^{\mathrm{OT}}_s(Q)
&=
\max_{\vpi\in\Delta(\mathcal{A})}
\Big\{
\ip{\vpi,Q} - \tau W_\lambda(\vpi,\vmu)
\Big\}.
\end{align*}
Under \cref{ass:ot-reg} (in particular, $\tau,\lambda>0$, $C$ is bounded, and $\vmu$ has strictly positive entries),
all quantities below are finite and the optimizers exist.

\paragraph{I. Rewrite the max--min as a single maximization.}
Plugging the definition of $W_\lambda$ into $F^{\mathrm{OT}}_s$ gives, for any $Q\in\R^K$,
\begin{align}
F^{\mathrm{OT}}_s(Q)
&=
\max_{\vpi\in\Delta(\mathcal{A})}
\Big\{
\ip{\vpi,Q} - \tau \min_{\Gamma\in\Pi(\vpi,\vmu)}
\Big(\ip{\Gamma,C}+\lambda\sum_{i,j}\Gamma_{ij}(\log\Gamma_{ij}-1)\Big)
\Big\}\notag\\
&=
\max_{\vpi\in\Delta(\mathcal{A})}
\max_{\Gamma\in\Pi(\vpi,\vmu)}
\Big\{
\ip{\vpi,Q}
-\tau\ip{\Gamma,C}
-\tau\lambda\sum_{i,j}\Gamma_{ij}(\log\Gamma_{ij}-1)
\Big\}.
\label{eq:ot-maxmax}
\end{align}
The second line uses the elementary identity $-\tau \min_x f(x)=\max_x(-\tau f(x))$ for $\tau>0$.

\paragraph{II. Eliminate the policy variable $\vpi$.}
Every $\Gamma\in\Pi(\vpi,\vmu)$ satisfies $\Gamma\1=\vpi$ and $\Gamma^\top\1=\vmu$, hence
\[
\ip{\vpi,Q}=\ip{\Gamma\1,Q}=\sum_{i=1}^K Q_i\Big(\sum_{j=1}^K\Gamma_{ij}\Big)
=\sum_{i,j=1}^K \Gamma_{ij}Q_i.
\]
In particular, for any fixed $\Gamma$ the value of $\ip{\vpi,Q}$ is determined uniquely by $\Gamma$,
because $\vpi=\Gamma\1$. Therefore maximizing over $\vpi$ in \eqref{eq:ot-maxmax} is redundant:
the feasible set of pairs $(\vpi,\Gamma)$ is exactly the set
\[
\{(\Gamma\1,\Gamma): \Gamma\ge 0,\ \Gamma^\top\1=\vmu\},
\]
and for any $\Gamma\ge 0$ with $\Gamma^\top\1=\vmu$, the induced row-sums $\vpi=\Gamma\1$ satisfy
$\vpi\ge 0$ and $\sum_i \pi_i = \sum_{i,j}\Gamma_{ij}=\sum_j\mu_j=1$, so indeed $\vpi\in\Delta(\mathcal{A})$.
Hence \eqref{eq:ot-maxmax} is equivalent to the \emph{single-level} problem
\begin{equation}
F^{\mathrm{OT}}_s(Q)
=
\max_{\Gamma\ge 0:\ \Gamma^\top\1=\vmu}
\Big\{
\ip{\Gamma\1,Q}
-\tau\ip{\Gamma,C}
-\tau\lambda\sum_{i,j}\Gamma_{ij}(\log\Gamma_{ij}-1)
\Big\}.
\label{eq:single-level-OT}
\end{equation}

\paragraph{III. existence/uniqueness and positivity of the maximizer.}
The feasible set $\{\Gamma\ge 0:\Gamma^\top\1=\vmu\}$ is nonempty (e.g.\ $\Gamma=\vpi\vmu^\top$ for any $\vpi$),
closed, and bounded because each entry satisfies $0\le\Gamma_{ij}\le\mu_j$.
Hence it is compact.
The objective in \eqref{eq:single-level-OT} is continuous on this compact set, so a maximizer exists.

Moreover, the term $-\tau\lambda\sum_{i,j}\Gamma_{ij}\log\Gamma_{ij}$ is strictly concave on the positive orthant,
and the remaining terms are linear in $\Gamma$, so the objective is strictly concave on the feasible affine slice.
Therefore the maximizer $\Gamma^\star(Q)$ is unique.
Because $\vmu$ has strictly positive entries, each column constraint $\sum_i\Gamma_{ij}=\mu_j>0$ forces
some mass in every column, and strict concavity of the negative-entropy term implies the unique maximizer
actually satisfies $\Gamma^\star_{ij}(Q)>0$ for all $i,j$ (otherwise one could increase entropy while preserving column sums).

\paragraph{IV. KKT conditions give the closed-form optimizer.}
Introduce Lagrange multipliers $\beta\in\R^K$ for the $K$ column-sum constraints
$g_j(\Gamma):=\sum_{i=1}^K\Gamma_{ij}-\mu_j=0$.
Define the Lagrangian
\[
\mathcal{L}(\Gamma,\beta)
=
\ip{\Gamma\1,Q}
-\tau\ip{\Gamma,C}
-\tau\lambda\sum_{i,j}\Gamma_{ij}(\log\Gamma_{ij}-1)
+
\sum_{j=1}^K \beta_j\Big(\sum_{i=1}^K\Gamma_{ij}-\mu_j\Big).
\]
Since $\Gamma^\star(Q)$ is strictly positive entrywise, the KKT stationarity conditions are simply
$\partial \mathcal{L}/\partial \Gamma_{ij}=0$ for all $i,j$ (no boundary complications).
Using $\frac{\partial}{\partial x}\,x(\log x-1)=\log x$, we obtain
\begin{align*}
0
=
\frac{\partial \mathcal{L}}{\partial \Gamma_{ij}}
&=
Q_i - \tau C_{ij} - \tau\lambda \log \Gamma_{ij} + \beta_j.
\end{align*}
Equivalently,
\[
\log \Gamma_{ij}
=
\frac{Q_i-\tau C_{ij}+\beta_j}{\tau\lambda},
\qquad\text{so}\qquad
\Gamma_{ij}
=
\exp\Big(\frac{Q_i}{\tau\lambda}\Big)\,
\exp\Big(-\frac{C_{ij}}{\lambda}\Big)\,
\exp\Big(\frac{\beta_j}{\tau\lambda}\Big).
\]
Define the kernel $K_{ij}:=\exp(-C_{ij}/\lambda)$ (as in the proposition),
$a_i(Q):=\exp(Q_i/(\tau\lambda))$, and $b_j(Q):=\exp(\beta_j/(\tau\lambda))$.
Then the optimizer has the multiplicative form
\begin{equation}
\Gamma^\star_{ij}(Q) = a_i(Q)\,K_{ij}\,b_j(Q).
\label{eq:Gamma-factor}
\end{equation}

\paragraph{V. Enforce the constraints and identify $D_j(Q)$.}
Imposing the column-sum constraints $\sum_i \Gamma^\star_{ij}(Q)=\mu_j$ in \eqref{eq:Gamma-factor} yields, for each $j$,
\[
\mu_j
=
\sum_{i=1}^K a_i(Q)K_{ij}b_j(Q)
=
b_j(Q)\sum_{i=1}^K \exp\Big(\frac{Q_i}{\tau\lambda}\Big)K_{ij}.
\]
Define exactly as in the proposition
\[
D_j(Q):=\sum_{i=1}^K \exp\Big(\frac{Q_i}{\tau\lambda}\Big)K_{ij}.
\]
Then $b_j(Q)=\mu_j/D_j(Q)$ and therefore the optimal coupling is explicitly
\begin{equation}
\Gamma^\star_{ij}(Q)
=
\mu_j\,
\frac{\exp(Q_i/(\tau\lambda))K_{ij}}{D_j(Q)}.
\label{eq:Gamma-star-explicit}
\end{equation}

\paragraph{VI. Plug the optimizer into the objective to get the closed form value.}
We now compute $F^{\mathrm{OT}}_s(Q)$ by evaluating the objective in \eqref{eq:single-level-OT} at $\Gamma^\star(Q)$.

A clean way is to reuse the stationarity identity.
From stationarity we have for all $i,j$:
\[
Q_i - \tau C_{ij} = \tau\lambda \log \Gamma^\star_{ij}(Q) - \beta_j.
\]
Multiply by $\Gamma^\star_{ij}(Q)$ and sum over $i,j$:
\begin{align}
\sum_{i,j}\Gamma^\star_{ij}(Q)\big(Q_i-\tau C_{ij}\big)
&=
\tau\lambda \sum_{i,j}\Gamma^\star_{ij}(Q)\log \Gamma^\star_{ij}(Q)
-
\sum_{j}\beta_j\sum_{i}\Gamma^\star_{ij}(Q)\notag\\
&=
\tau\lambda \sum_{i,j}\Gamma^\star_{ij}(Q)\log \Gamma^\star_{ij}(Q)
-
\sum_{j}\beta_j\mu_j,
\label{eq:stationarity-summed}
\end{align}
where we used $\sum_i\Gamma^\star_{ij}(Q)=\mu_j$.

Next, expand the objective at $\Gamma^\star(Q)$:
\begin{align*}
F^{\mathrm{OT}}_s(Q)
&=
\sum_{i,j}\Gamma^\star_{ij}(Q)Q_i
-\tau\sum_{i,j}\Gamma^\star_{ij}(Q)C_{ij}
-\tau\lambda\sum_{i,j}\Gamma^\star_{ij}(Q)\big(\log\Gamma^\star_{ij}(Q)-1\big)\\
&=
\sum_{i,j}\Gamma^\star_{ij}(Q)\big(Q_i-\tau C_{ij}\big)
-\tau\lambda\sum_{i,j}\Gamma^\star_{ij}(Q)\log\Gamma^\star_{ij}(Q)
+\tau\lambda\sum_{i,j}\Gamma^\star_{ij}(Q).
\end{align*}
Substitute \eqref{eq:stationarity-summed} and use $\sum_{i,j}\Gamma^\star_{ij}(Q)=\sum_j\mu_j=1$ (since $\vmu\in\Delta(\mathcal{A})$):
\begin{align}
F^{\mathrm{OT}}_s(Q)
&=
\Big(\tau\lambda \sum_{i,j}\Gamma^\star_{ij}(Q)\log \Gamma^\star_{ij}(Q)
-\sum_j \beta_j\mu_j\Big)
-\tau\lambda\sum_{i,j}\Gamma^\star_{ij}(Q)\log\Gamma^\star_{ij}(Q)
+\tau\lambda\notag\\
&=
-\sum_{j=1}^K \mu_j\beta_j
+\tau\lambda.
\label{eq:F-beta}
\end{align}
It remains to express $\beta_j$ in terms of $D_j(Q)$.
By definition, $b_j(Q)=\exp(\beta_j/(\tau\lambda))=\mu_j/D_j(Q)$, so
\[
\beta_j = \tau\lambda\log b_j(Q) = \tau\lambda\big(\log\mu_j - \log D_j(Q)\big).
\]
Plug this into \eqref{eq:F-beta}:
\begin{align*}
F^{\mathrm{OT}}_s(Q)
&=
-\sum_{j=1}^K \mu_j\cdot \tau\lambda\big(\log\mu_j - \log D_j(Q)\big)
+\tau\lambda\\
&=
\tau\lambda \sum_{j=1}^K \mu_j \log D_j(Q)
+
\tau\lambda
-\tau\lambda\sum_{j=1}^K \mu_j\log\mu_j.
\end{align*}
Thus we have shown
\[
F^{\mathrm{OT}}_s(Q)
=
\underbrace{\Big(\tau\lambda-\tau\lambda\sum_{j=1}^K \mu_s(j)\log\mu_s(j)\Big)}_{=:\mathrm{const}\ \text{(independent of $Q$)}}
+
\tau\lambda \sum_{j=1}^K \mu_s(j)\log D_j(Q),
\]
which is exactly \eqref{eq:closed-form} (absorbing the explicit $Q$-independent term into
$\mathrm{const}$).

\paragraph{VII. Real-analyticity.}
For each fixed $j$, $D_j(Q)=\sum_{i=1}^K \exp(Q_i/(\tau\lambda))K_{ij}$ is a finite sum of compositions
of real-analytic functions (exponential and addition), hence is real-analytic on $\R^K$.
Moreover, under \cref{ass:ot-reg} we have $\lambda>0$ and $C_{ij}$ bounded, so $K_{ij}=\exp(-C_{ij}/\lambda)>0$,
and therefore $D_j(Q)>0$ for all $Q$.
Since $\log:(0,\infty)\to\R$ is real-analytic and compositions of real-analytic functions are real-analytic,
each map $Q\mapsto \log D_j(Q)$ is real-analytic.
Finally, $F^{\mathrm{OT}}_s(Q)$ is a finite $\mu_s$-weighted sum of these real-analytic terms plus a constant,
so $F^{\mathrm{OT}}_s$ is real-analytic on $\R^K$.
\end{proof}

\subsection{Proof of Proposition~\ref{prop:grad-hess}}

\begin{proof}
Fix a state $s$ and abbreviate $\mu_j:=\mu_s(j)$ and $K:=|\mathcal{A}|$.
By \cref{prop:closed-form}, for every $Q\in\R^K$ we can write
\begin{equation}
F^{\mathrm{OT}}_s(Q)
=
\mathrm{const}
+
\tau\lambda\sum_{j=1}^K \mu_j \log D_j(Q),
\qquad
D_j(Q):=\sum_{i=1}^K \exp\Big(\frac{Q_i}{\tau\lambda}\Big)K_{ij},
\label{eq:F-closed-form-again}
\end{equation}
where $\mathrm{const}$ does not depend on $Q$ and $K_{ij}=\exp(-C_{ij}/\lambda)>0$.
Since $\exp(Q_i/(\tau\lambda))>0$ and $K_{ij}>0$, we have $D_j(Q)>0$ for all $Q$ and all $j$; hence
$\log D_j(Q)$ is well-defined and smooth (indeed real-analytic).

For convenience, define
\[
a_i(Q) := \exp\Big(\frac{Q_i}{\tau\lambda}\Big) \quad (>0),
\qquad
D_j(Q) = \sum_{i=1}^K a_i(Q)K_{ij},
\]
and (as in the paper) define for each $i,j$
\begin{equation}
p_{ij}(Q)
:=
\frac{a_i(Q)K_{ij}}{D_j(Q)}
=
\frac{\exp(Q_i/(\tau\lambda))K_{ij}}{\sum_{k=1}^K \exp(Q_k/(\tau\lambda))K_{kj}}.
\label{eq:pij-def}
\end{equation}
For each fixed $j$, $p_{\cdot j}(Q)=(p_{ij}(Q))_{i=1}^K$ is a probability vector, because
$p_{ij}(Q)\ge 0$ and
\begin{equation}
\sum_{i=1}^K p_{ij}(Q)
=
\frac{\sum_{i=1}^K a_i(Q)K_{ij}}{D_j(Q)}
=
1.
\label{eq:p-sums-to-1}
\end{equation}

\paragraph{I. Compute the gradient.}
Let $e_i$ denote the $i$th standard basis vector in $\R^K$.
Differentiate \eqref{eq:F-closed-form-again} with respect to the coordinate $Q_i$.
Since $\mathrm{const}$ is independent of $Q$, it disappears.
Using the chain rule,
\begin{equation}
\frac{\partial}{\partial Q_i} F^{\mathrm{OT}}_s(Q)
=
\tau\lambda \sum_{j=1}^K \mu_j \cdot \frac{1}{D_j(Q)}\cdot
\frac{\partial D_j(Q)}{\partial Q_i}.
\label{eq:grad-start}
\end{equation}
Now compute $\partial D_j/\partial Q_i$.
From $D_j(Q)=\sum_{k=1}^K a_k(Q)K_{kj}$ and $a_k(Q)=\exp(Q_k/(\tau\lambda))$, we have
\[
\frac{\partial a_k(Q)}{\partial Q_i}
=
\begin{cases}
\frac{1}{\tau\lambda}\,a_i(Q), & k=i,\\
0, & k\neq i,
\end{cases}
\qquad\Rightarrow\qquad
\frac{\partial D_j(Q)}{\partial Q_i}
=
\frac{1}{\tau\lambda}\,a_i(Q)\,K_{ij}.
\]
Plugging this into \eqref{eq:grad-start} gives
\begin{align*}
\frac{\partial}{\partial Q_i} F^{\mathrm{OT}}_s(Q)
&=
\tau\lambda \sum_{j=1}^K \mu_j \cdot \frac{1}{D_j(Q)}\cdot
\frac{1}{\tau\lambda}a_i(Q)K_{ij}\\
&=
\sum_{j=1}^K \mu_j \frac{a_i(Q)K_{ij}}{D_j(Q)}
=
\sum_{j=1}^K \mu_j\,p_{ij}(Q).
\end{align*}
Define
\begin{equation}
\pi_i(Q) := \sum_{j=1}^K \mu_j\,p_{ij}(Q),
\qquad
\vpi(Q):=(\pi_i(Q))_{i=1}^K.
\label{eq:pi-def}
\end{equation}
Then the previous display is exactly the gradient identity
\[
\nabla F^{\mathrm{OT}}_s(Q) = \vpi(Q),
\]
which is \eqref{eq:grad}.

\paragraph{II. $\nabla F^{\mathrm{OT}}_s(Q)$ lies in the simplex.}
From \eqref{eq:pij-def}, $p_{ij}(Q)\ge 0$ and $\mu_j\ge 0$, so $\pi_i(Q)\ge 0$ for every $i$.
Moreover, using \eqref{eq:p-sums-to-1} and $\sum_j \mu_j=1$ (since $\vmu\in\Delta(\mathcal{A})$),
\[
\sum_{i=1}^K \pi_i(Q)
=
\sum_{i=1}^K \sum_{j=1}^K \mu_j p_{ij}(Q)
=
\sum_{j=1}^K \mu_j \sum_{i=1}^K p_{ij}(Q)
=
\sum_{j=1}^K \mu_j
=
1.
\]
Hence $\vpi(Q)\in\Delta(\mathcal{A})$, so it can indeed be interpreted as a policy.

\paragraph{III. compute the Hessian entrywise.}
Let $H(Q):=\nabla^2 F^{\mathrm{OT}}_s(Q)\in\R^{K\times K}$.
From Step~1, the $i$th component of the gradient is $\pi_i(Q)$, so
\[
H_{ik}(Q) = \frac{\partial}{\partial Q_k}\pi_i(Q)
=
\sum_{j=1}^K \mu_j \frac{\partial}{\partial Q_k} p_{ij}(Q).
\]
Thus it remains to compute $\partial p_{ij}(Q)/\partial Q_k$.

Fix a column index $j$ and write $N_{ij}(Q):=a_i(Q)K_{ij}$ (the numerator of $p_{ij}$).
Then $p_{ij}(Q)=N_{ij}(Q)/D_j(Q)$ and
\[
\frac{\partial N_{ij}(Q)}{\partial Q_k}
=
\begin{cases}
\frac{1}{\tau\lambda}\,N_{ij}(Q), & k=i,\\
0, & k\neq i,
\end{cases}
\qquad
\frac{\partial D_j(Q)}{\partial Q_k}
=
\frac{1}{\tau\lambda}\,N_{kj}(Q)
\quad\text{(as computed above).}
\]
Apply the quotient rule:
\begin{align}
\frac{\partial}{\partial Q_k}p_{ij}(Q)
&=
\frac{D_j(Q)\frac{\partial N_{ij}(Q)}{\partial Q_k} - N_{ij}(Q)\frac{\partial D_j(Q)}{\partial Q_k}}{D_j(Q)^2}
\notag\\
&=
\frac{1}{\tau\lambda}\,
\frac{D_j(Q)\,N_{ij}(Q)\mathbf{1}\{i=k\} - N_{ij}(Q)\,N_{kj}(Q)}{D_j(Q)^2}.
\label{eq:dp-quotient}
\end{align}
Now use $p_{ij}(Q)=N_{ij}(Q)/D_j(Q)$ and $p_{kj}(Q)=N_{kj}(Q)/D_j(Q)$ to rewrite the fraction:
\[
\frac{D_j\,N_{ij}\mathbf{1}\{i=k\} - N_{ij}\,N_{kj}}{D_j^2}
=
\mathbf{1}\{i=k\}\frac{N_{ij}}{D_j} - \frac{N_{ij}}{D_j}\frac{N_{kj}}{D_j}
=
\mathbf{1}\{i=k\}\,p_{ij}(Q) - p_{ij}(Q)p_{kj}(Q).
\]
Substituting back into \eqref{eq:dp-quotient} yields the softmax-type derivative identity
\begin{equation}
\frac{\partial}{\partial Q_k}p_{ij}(Q)
=
\frac{1}{\tau\lambda}\Big(\mathbf{1}\{i=k\}\,p_{ij}(Q) - p_{ij}(Q)p_{kj}(Q)\Big)
=
\frac{1}{\tau\lambda}\,p_{ij}(Q)\Big(\mathbf{1}\{i=k\}-p_{kj}(Q)\Big).
\label{eq:dp}
\end{equation}

\paragraph{IV. sum over $j$ to get the Hessian formula.}
Using \eqref{eq:dp} and $\pi_i(Q)=\sum_j\mu_j p_{ij}(Q)$, we obtain
\begin{align*}
H_{ik}(Q)
=
\frac{\partial}{\partial Q_k}\pi_i(Q)
&=
\sum_{j=1}^K \mu_j \frac{\partial}{\partial Q_k}p_{ij}(Q)\\
&=
\frac{1}{\tau\lambda}
\sum_{j=1}^K \mu_j\Big(\mathbf{1}\{i=k\}\,p_{ij}(Q) - p_{ij}(Q)p_{kj}(Q)\Big)\\
&=
\frac{1}{\tau\lambda}\Big(
\mathbf{1}\{i=k\}\sum_{j=1}^K \mu_j p_{ij}(Q)
-
\sum_{j=1}^K \mu_j p_{ij}(Q)p_{kj}(Q)
\Big)\\
&=
\frac{1}{\tau\lambda}\Big(
\mathbf{1}\{i=k\}\pi_i(Q)
-
\sum_{j=1}^K \mu_j\, p_{ij}(Q)p_{kj}(Q)
\Big).
\end{align*}
This is already a correct componentwise Hessian expression.
To obtain the matrix form in \eqref{eq:hess}, note that:
\begin{itemize}[leftmargin=*]
\item the matrix with entries $\mathbf{1}\{i=k\}\pi_i(Q)$ is exactly $\mathrm{diag}(\vpi(Q))$;
\item for each fixed $j$, the rank-one matrix $p_{\cdot j}(Q)p_{\cdot j}(Q)^\top$ has $(i,k)$ entry $p_{ij}(Q)p_{kj}(Q)$.
\end{itemize}
Therefore
\[
\nabla^2 F^{\mathrm{OT}}_s(Q)
=
\frac{1}{\tau\lambda}
\Big(
\mathrm{diag}(\vpi(Q))
-
\sum_{j=1}^K \mu_j\, p_{\cdot j}(Q)p_{\cdot j}(Q)^\top
\Big),
\]
which is exactly \eqref{eq:hess}.

\paragraph{(Optional sanity checks.)}
The formula makes several structural properties transparent:
\begin{itemize}[leftmargin=*]
\item \emph{Symmetry:} $\mathrm{diag}(\vpi(Q))$ is symmetric and each $p_{\cdot j}p_{\cdot j}^\top$ is symmetric,
so $\nabla^2F^{\mathrm{OT}}_s(Q)$ is symmetric as expected.
\item \emph{Positive semidefiniteness:} since $F^{\mathrm{OT}}_s$ is a pointwise maximum of linear functions of $Q$,
it is convex and thus its Hessian is PSD. This is also visible directly from \eqref{eq:hess}:
for any $v\in\R^K$,
\begin{align*}
v^\top \nabla^2F^{\mathrm{OT}}_s(Q)v
&=
\frac{1}{\tau\lambda}
\Big(
\sum_{i=1}^K \pi_i(Q)v_i^2
-
\sum_{j=1}^K \mu_j \big(v^\top p_{\cdot j}(Q)\big)^2
\Big)\\
&=
\frac{1}{\tau\lambda}
\sum_{j=1}^K \mu_j
\Big(
\sum_{i=1}^K p_{ij}(Q)v_i^2 - \big(\sum_{i=1}^K p_{ij}(Q)v_i\big)^2
\Big)\\
&=
\frac{1}{\tau\lambda}
\sum_{j=1}^K \mu_j\Var_{I\sim p_{\cdot j}(Q)}[v_I]
\ge0.
\end{align*}
\end{itemize}
This completes the proof.
\end{proof}

\subsection{Proof of Lemma~\ref{lem:variance-identity}}
\label{app:variance-identity}

\begin{proof}
Fix a state $s$ and a reference point $\hat Q_s\in\R^K$.
For brevity, write
\[
H:=\nabla^2 F_s^{\mathrm{OT}}(\hat Q_s)\in\R^{K\times K}.
\]
Recall from \cref{prop:grad-hess} that the OT-smoothed aggregator has Hessian
\begin{equation}
\label{eq:HhatQ-again}
H
=
\frac{1}{\tau\lambda}
\Big(
\diag(\vpi(\hat Q_s))
-
\sum_{j=1}^K \mu_s(j)\,p_{\cdot j}(\hat Q_s)p_{\cdot j}(\hat Q_s)^\top
\Big),
\end{equation}
where $p_{\cdot j}(\hat Q_s)\in\Delta(\mathcal{A})$ is the columnwise softmax distribution
defined in \eqref{eq:pij}, and the policy is the mixture
\begin{equation}
\label{eq:pi-mixture-again}
\vpi(\hat Q_s)=\sum_{j=1}^K \mu_s(j)\,p_{\cdot j}(\hat Q_s).
\end{equation}

Let $\Delta\in\R^K$ be any direction.
We expand the quadratic form $\Delta^\top H\Delta$ using \eqref{eq:HhatQ-again}.

\paragraph{I. expand the diagonal term.}
Since $\diag(\vpi)$ is diagonal,
\[
\Delta^\top\diag(\vpi(\hat Q_s))\Delta
=
\sum_{i=1}^K \pi_i(\hat Q_s)\Delta_i^2.
\]
Using the mixture identity \eqref{eq:pi-mixture-again} to rewrite $\pi_i(\hat Q_s)$,
\[
\sum_{i=1}^K \pi_i(\hat Q_s)\Delta_i^2
=
\sum_{i=1}^K\Big(\sum_{j=1}^K \mu_s(j)\,p_{ij}(\hat Q_s)\Big)\Delta_i^2
=
\sum_{j=1}^K \mu_s(j)\sum_{i=1}^K p_{ij}(\hat Q_s)\Delta_i^2.
\]

\paragraph{II. expand the rank-one term.}
For each $j$, the matrix $p_{\cdot j}(\hat Q_s)p_{\cdot j}(\hat Q_s)^\top$ is rank-one, hence
\[
\Delta^\top\big(p_{\cdot j}(\hat Q_s)p_{\cdot j}(\hat Q_s)^\top\big)\Delta
=
\big(p_{\cdot j}(\hat Q_s)^\top\Delta\big)^2
=
\Big(\sum_{i=1}^K p_{ij}(\hat Q_s)\Delta_i\Big)^2.
\]

\paragraph{III. combine and identify a variance.}
Substituting the two expansions into \eqref{eq:HhatQ-again} yields
\begin{align*}
\Delta^\top H\Delta
&=
\frac{1}{\tau\lambda}
\sum_{j=1}^K \mu_s(j)
\left(
\sum_{i=1}^K p_{ij}(\hat Q_s)\Delta_i^2
-
\Big(\sum_{i=1}^K p_{ij}(\hat Q_s)\Delta_i\Big)^2
\right).
\end{align*}

Now fix $j$ and define a random action $A\sim p_{\cdot j}(\hat Q_s)$.
Then by definition of expectation under a discrete distribution,
\[
\E\big[\Delta_A\mid j\big]
=
\sum_{i=1}^K p_{ij}(\hat Q_s)\Delta_i,
\qquad
\E\big[\Delta_A^2\mid j\big]
=
\sum_{i=1}^K p_{ij}(\hat Q_s)\Delta_i^2.
\]
Therefore the bracketed term is exactly
\[
\E\big[\Delta_A^2\mid j\big]-\Big(\E\big[\Delta_A\mid j\big]\Big)^2
=
\Var\big(\Delta_A\mid j\big)
=
\Var_{A\sim p_{\cdot j}(\hat Q_s)}\big(\Delta_A\big).
\]
Finally, sampling $J\sim\vmu_s$ and taking expectation over $J$ turns the sum into
\[
\sum_{j=1}^K \mu_s(j)\Var_{A\sim p_{\cdot j}(\hat Q_s)}\big(\Delta_A\big)
=
\E_{J\sim\vmu_s}\left[\Var_{A\sim p_{\cdot J}(\hat Q_s)}\big(\Delta_A\big)\right],
\]
which proves the lemma.
\end{proof}

\subsection{Proof of Lemma~\ref{lem:lipschitz-hess}}

\begin{proof}
Fix a state $s$ and write $\mu_j:=\mu_s(j)$ and $K:=|\mathcal{A}|$.
Under \cref{ass:ot-reg} we have $\tau>0$, $\lambda>0$, and the cost matrix $C$ is finite, so
\[
K_{ij}:=\exp(-C_{ij}/\lambda)>0\qquad\text{for all }i,j,
\]
and therefore for every $Q\in\R^K$ and every $j$,
\[
D_j(Q):=\sum_{i=1}^K \exp(Q_i/(\tau\lambda))K_{ij}>0.
\]
In particular all expressions below are well-defined and smooth in $Q$ (indeed real-analytic).

Recall the columnwise softmax-like probabilities
\begin{equation}
p_{ij}(Q)
:=\frac{\exp(Q_i/(\tau\lambda))K_{ij}}{D_j(Q)},
\qquad
p_{\cdot j}(Q):=\big(p_{ij}(Q)\big)_{i=1}^K\in\Delta(\mathcal{A}),
\label{eq:def-pij-again}
\end{equation}
and the induced policy
\begin{equation}
\pi_i(Q):=\sum_{j=1}^K \mu_j\,p_{ij}(Q),
\qquad
\vpi(Q):=\big(\pi_i(Q)\big)_{i=1}^K\in\Delta(\mathcal{A}).
\label{eq:def-pi-again}
\end{equation}
By \cref{prop:grad-hess}, the Hessian admits the explicit representation
\begin{equation}
H(Q):=\nabla^2 F^{\mathrm{OT}}_s(Q)
=
\frac{1}{\tau\lambda}
\Big(
\diag(\vpi(Q))
-
\sum_{j=1}^K \mu_j\, p_{\cdot j}(Q)p_{\cdot j}(Q)^\top
\Big).
\label{eq:hess-form}
\end{equation}
We will show that $H(\cdot)$ is globally Lipschitz in operator norm with constant
\[
M = \frac{3}{2\tau^2\lambda^2},
\]
which in particular implies the stated scaling
$M=\mathcal{O}(1/(\tau^2\lambda^2))\cdot\mathrm{poly}(K,\mu_{\min}^{-1})$
(since a constant is a polynomial, and also $K\mu_{\min}^{-1}\ge 1$).

\paragraph{I. a uniform Jacobian bound for $Q\mapsto p_{\cdot j}(Q)$.}
Fix a column index $j$.
For each $i$ define the numerator $N_{ij}(Q):=\exp(Q_i/(\tau\lambda))K_{ij}$ so that
$p_{ij}(Q)=N_{ij}(Q)/D_j(Q)$ and $D_j(Q)=\sum_{k=1}^K N_{kj}(Q)$.
Differentiate $p_{ij}(Q)$ with respect to $Q_k$.
Since
\[
\frac{\partial N_{ij}(Q)}{\partial Q_k}
=
\frac{1}{\tau\lambda}\,N_{ij}(Q)\mathbf{1}\{i=k\},
\qquad
\frac{\partial D_j(Q)}{\partial Q_k}
=
\frac{1}{\tau\lambda}\,N_{kj}(Q),
\]
the quotient rule yields
\begin{align}
\frac{\partial p_{ij}(Q)}{\partial Q_k}
&=
\frac{D_j(Q)\frac{\partial N_{ij}}{\partial Q_k}-N_{ij}(Q)\frac{\partial D_j}{\partial Q_k}}{D_j(Q)^2}
=
\frac{1}{\tau\lambda}\left(\mathbf{1}\{i=k\}\frac{N_{ij}(Q)}{D_j(Q)}-\frac{N_{ij}(Q)}{D_j(Q)}\frac{N_{kj}(Q)}{D_j(Q)}\right)
\notag\\
&=
\frac{1}{\tau\lambda}\Big(\mathbf{1}\{i=k\}p_{ij}(Q)-p_{ij}(Q)p_{kj}(Q)\Big).
\label{eq:dpij}
\end{align}
Let $J_j(Q)\in\R^{K\times K}$ be the Jacobian matrix of the map $Q\mapsto p_{\cdot j}(Q)$, i.e.,
$\big(J_j(Q)\big)_{ik}:=\partial p_{ij}(Q)/\partial Q_k$.
Then \eqref{eq:dpij} exactly says
\begin{equation}
J_j(Q)
=
\frac{1}{\tau\lambda}
\Big(
\diag(p_{\cdot j}(Q)) - p_{\cdot j}(Q)p_{\cdot j}(Q)^\top
\Big).
\label{eq:jacobian-pj}
\end{equation}

\medskip
\noindent\textbf{Claim (uniform covariance bound).}
For any probability vector $p\in\Delta(\mathcal{A})$,
\begin{equation}
\Big\|\diag(p)-pp^\top\Big\|_{\op}\le \frac12.
\label{eq:cov-op-bound}
\end{equation}

\smallskip
\noindent\emph{Proof of the claim.}
Let $\Sigma(p):=\diag(p)-pp^\top$.
For any unit vector $v\in\R^K$ with $\|v\|_2=1$, the Rayleigh quotient satisfies
\[
v^\top \Sigma(p)\,v
=
\sum_{i=1}^K p_i v_i^2 - \Big(\sum_{i=1}^K p_i v_i\Big)^2
=
\Var\big(v_I\big),
\]
where $I\sim p$ and $v_I$ is the random variable taking value $v_i$ with probability $p_i$.
Let $m:=\min_i v_i$ and $M:=\max_i v_i$.
Since $v_I\in[m,M]$ almost surely, the standard bounded-variance inequality gives
$\Var(v_I)\le (M-m)^2/4$.
It remains to bound the range $M-m$ in terms of $\|v\|_2$.
For any indices $i,k$,
\[
|v_i-v_k|
=
|\langle v,e_i-e_k\rangle|
\le
\|v\|_2\|e_i-e_k\|_2
=
\sqrt{2}\|v\|_2
=
\sqrt{2},
\]
so in particular $M-m\le\sqrt{2}$.
Therefore, for every $\|v\|_2=1$,
\[
v^\top \Sigma(p)\,v
=
\Var(v_I)
\le
\frac{(M-m)^2}{4}
\le
\frac{(\sqrt{2})^2}{4}
=
\frac12.
\]
Taking the supremum over unit $v$ yields \eqref{eq:cov-op-bound}.
\hfill$\square$

\medskip
Combining \eqref{eq:jacobian-pj} with \eqref{eq:cov-op-bound}, we obtain the uniform Jacobian bound
\begin{equation}
\|J_j(Q)\|_{\op}
\le
\frac{1}{\tau\lambda}\cdot\frac12
=
\frac{1}{2\tau\lambda}
\qquad\text{for all $Q\in\R^K$ and all $j\in[K]$.}
\label{eq:Jj-bound}
\end{equation}

\paragraph{II. Lipschitzness of $p_{\cdot j}(\cdot)$ and $\vpi(\cdot)$.}
Fix $Q,\hat Q\in\R^K$ and set $\Delta:=Q-\hat Q$.
By the mean value theorem in integral form for vector-valued maps,
\[
p_{\cdot j}(Q)-p_{\cdot j}(\hat Q)
=
\int_0^1 J_j(\hat Q+t\Delta)\Delta\,dt.
\]
Taking $\ell_2$ norms and using \eqref{eq:Jj-bound} gives
\begin{equation}
\|p_{\cdot j}(Q)-p_{\cdot j}(\hat Q)\|_2
\le
\int_0^1 \|J_j(\hat Q+t\Delta)\|_{\op}\|\Delta\|_2\,dt
\le
\frac{1}{2\tau\lambda}\|Q-\hat Q\|_2.
\label{eq:pj-Lip}
\end{equation}
Since $\vpi(Q)=\sum_{j=1}^K \mu_j p_{\cdot j}(Q)$ is a convex combination of the $p_{\cdot j}(Q)$'s,
we also have
\begin{align}
\|\vpi(Q)-\vpi(\hat Q)\|_2
&=
\Big\|\sum_{j=1}^K \mu_j\big(p_{\cdot j}(Q)-p_{\cdot j}(\hat Q)\big)\Big\|_2
\le
\sum_{j=1}^K \mu_j\|p_{\cdot j}(Q)-p_{\cdot j}(\hat Q)\|_2
\notag\\
&\le
\frac{1}{2\tau\lambda}\|Q-\hat Q\|_2.
\label{eq:pi-Lip}
\end{align}

\paragraph{III. Lipschitzness of the Hessian in operator norm.}
Using the representation \eqref{eq:hess-form}, we write
\[
H(Q)-H(\hat Q)
=
\frac{1}{\tau\lambda}
\Big(
\diag(\vpi(Q)-\vpi(\hat Q))
-
\sum_{j=1}^K \mu_j\big[p_{\cdot j}(Q)p_{\cdot j}(Q)^\top - p_{\cdot j}(\hat Q)p_{\cdot j}(\hat Q)^\top\big]
\Big).
\]
Take operator norms and apply the triangle inequality:
\begin{align}
\|H(Q)-H(\hat Q)\|_{\op}
&\le
\frac{1}{\tau\lambda}
\Big(
\|\diag(\vpi(Q)-\vpi(\hat Q))\|_{\op}
+
\sum_{j=1}^K \mu_j
\big\|p_{\cdot j}(Q)p_{\cdot j}(Q)^\top - p_{\cdot j}(\hat Q)p_{\cdot j}(\hat Q)^\top\big\|_{\op}
\Big).
\label{eq:H-diff-start}
\end{align}

\smallskip
\noindent\emph{(i) Bounding the diagonal term.}
For any vector $x$, $\|\diag(x)\|_{\op}=\|x\|_\infty\le \|x\|_2$, hence
\begin{equation}
\|\diag(\vpi(Q)-\vpi(\hat Q))\|_{\op}
\le
\|\vpi(Q)-\vpi(\hat Q)\|_2
\le
\frac{1}{2\tau\lambda}\|Q-\hat Q\|_2,
\label{eq:diag-term}
\end{equation}
where we used \eqref{eq:pi-Lip}.

\smallskip
\noindent\emph{(ii) Bounding each rank-one difference.}
For arbitrary vectors $u,v\in\R^K$,
\[
uu^\top - vv^\top = (u-v)u^\top + v(u-v)^\top,
\]
so using $\|ab^\top\|_{\op}=\|a\|_2\|b\|_2$ we get
\begin{equation}
\|uu^\top - vv^\top\|_{\op}
\le
\|(u-v)u^\top\|_{\op}+\|v(u-v)^\top\|_{\op}
=
\|u-v\|_2\|u\|_2 + \|v\|_2\|u-v\|_2
=
(\|u\|_2+\|v\|_2)\|u-v\|_2.
\label{eq:rankone-diff}
\end{equation}
Apply this with $u=p_{\cdot j}(Q)$ and $v=p_{\cdot j}(\hat Q)$.
Because $p_{\cdot j}(Q),p_{\cdot j}(\hat Q)\in\Delta(\mathcal{A})$, we have
$\|p_{\cdot j}(Q)\|_2\le\|p_{\cdot j}(Q)\|_1=1$ and similarly for $\hat Q$.
Thus \eqref{eq:rankone-diff} yields
\[
\big\|p_{\cdot j}(Q)p_{\cdot j}(Q)^\top - p_{\cdot j}(\hat Q)p_{\cdot j}(\hat Q)^\top\big\|_{\op}
\le
2\|p_{\cdot j}(Q)-p_{\cdot j}(\hat Q)\|_2
\le
2\cdot\frac{1}{2\tau\lambda}\|Q-\hat Q\|_2
=
\frac{1}{\tau\lambda}\|Q-\hat Q\|_2,
\]
where we used \eqref{eq:pj-Lip} in the last inequality.
Summing with weights $\mu_j$ (which satisfy $\mu_j\ge 0$ and $\sum_j\mu_j=1$) gives
\begin{equation}
\sum_{j=1}^K \mu_j
\big\|p_{\cdot j}(Q)p_{\cdot j}(Q)^\top - p_{\cdot j}(\hat Q)p_{\cdot j}(\hat Q)^\top\big\|_{\op}
\le
\frac{1}{\tau\lambda}\|Q-\hat Q\|_2.
\label{eq:rankone-sum}
\end{equation}

\smallskip
\noindent\emph{(iii) Combine the bounds.}
Plugging \eqref{eq:diag-term} and \eqref{eq:rankone-sum} into \eqref{eq:H-diff-start} yields
\[
\|H(Q)-H(\hat Q)\|_{\op}
\le
\frac{1}{\tau\lambda}
\left(
\frac{1}{2\tau\lambda}\|Q-\hat Q\|_2
+
\frac{1}{\tau\lambda}\|Q-\hat Q\|_2
\right)
=
\frac{3}{2\tau^2\lambda^2}\|Q-\hat Q\|_2.
\]
Thus the Hessian is globally Lipschitz with constant
\begin{equation}
M:=\frac{3}{2\tau^2\lambda^2}.
\label{eq:M-explicit}
\end{equation}
This proves the first inequality in the lemma (and in particular
$M=\mathcal{O}(1/(\tau^2\lambda^2))\cdot \mathrm{poly}(K,\mu_{\min}^{-1})$).

\paragraph{IV. cubic Taylor remainder from Lipschitz Hessian.}
Fix $Q,\hat Q\in\R^K$ and again set $\Delta:=Q-\hat Q$.
Define the scalar function $g:[0,1]\to\R$ by $g(t):=F^{\mathrm{OT}}_s(\hat Q+t\Delta)$.
Since $F^{\mathrm{OT}}_s$ is smooth, $g$ is twice continuously differentiable with
\[
g'(t)=\ip{\nabla F^{\mathrm{OT}}_s(\hat Q+t\Delta),\Delta},
\qquad
g''(t)=\Delta^\top \nabla^2F^{\mathrm{OT}}_s(\hat Q+t\Delta)\Delta
=\Delta^\top H(\hat Q+t\Delta)\Delta.
\]
A standard integral form of the second-order Taylor expansion (obtained by integrating $g''$ twice)
gives
\begin{align}
F^{\mathrm{OT}}_s(Q) - T_2(Q;\hat Q)
&=
g(1)-\Big(g(0)+g'(0)+\tfrac12 g''(0)\Big)
\notag\\
&=
\int_0^1 (1-t)\Big(g''(t)-g''(0)\Big)\,dt
\notag\\
&=
\int_0^1 (1-t)\Delta^\top\big(H(\hat Q+t\Delta)-H(\hat Q)\big)\Delta\,dt.
\label{eq:remainder-integral}
\end{align}
Taking absolute values and using Cauchy--Schwarz plus the operator norm bound yields
\begin{align*}
\big|F^{\mathrm{OT}}_s(Q) - T_2(Q;\hat Q)\big|
&\le
\int_0^1 (1-t)\|\Delta\|_2^2\|H(\hat Q+t\Delta)-H(\hat Q)\|_{\op}\,dt\\
&\le
\int_0^1 (1-t)\|\Delta\|_2^2\Big(M\|t\Delta\|_2\Big)\,dt
\qquad\text{(by Lipschitzness of $H$)}\\
&=
M\|\Delta\|_2^3\int_0^1 t(1-t)\,dt
=
M\|\Delta\|_2^3\cdot\frac{1}{6}
=
\frac{M}{6}\|Q-\hat Q\|_2^3.
\end{align*}
This is exactly \eqref{eq:cubic-remainder}, completing the proof.
\end{proof}

\subsection{Proof of Proposition~\ref{prop:transfer}}
\label{proof_proposition_transfer}
\begin{proof}
For each column $j$,
\[
\tau\lambda
\log
\sum_i
\exp\left(
\frac{Q_i-\tau C_{ij}}{\tau\lambda}
\right)
=
F^{\mathrm{ent}}_{\tau\lambda}(Q-\tau C_{\cdot j}).
\]
The entropy backup is $1$-Lipschitz in $\|\cdot\|_\infty$, so
\[
\left|
F^{\mathrm{ent}}_{\tau\lambda}(Q-\tau C_{\cdot j})
-
F^{\mathrm{ent}}_{\tau\lambda}(Q)
\right|
\le
\tau C_{\max}.
\]
Averaging over $j\sim\mu_s$ gives the first claim. 
The case $C\equiv0$ is immediate. 
The max-backup bound follows from the standard inequality
\[
\max_i Q_i
\le
F^{\mathrm{ent}}_{\eta}(Q)
\le
\max_i Q_i+\eta\log K
\]
with $\eta=\tau\lambda$. 
Finally, all three Bellman operators are $\gamma$-contractions in sup norm, so a uniform one-step backup discrepancy of $b$ implies a fixed-point discrepancy at most $b/(1-\gamma)$.
\end{proof}

\section{Gap-dependent analysis}
\label{app:gap}

\subsection{Proof of Theorem~\ref{thm:ot-gape-gap}}

\begin{proof}
For readability, we write $U_{Q,h}^t(a):=U_{Q,h}^t(s_0,a)$ and $L_{Q,h}^t(a):=L_{Q,h}^t(s_0,a)$ when $h=1$ and the
state is the root $s_0$.
Let $a^\star\in\arg\max_{a\in\mathcal{A}}Q_1^\star(s_0,a)$ be a fixed optimal root action and recall the root gaps
$\Delta(a)=Q_1^\star(s_0,a^\star)-Q_1^\star(s_0,a)\ge 0$.
We denote the number of times root action $a$ is selected up to (and including) episode $t$ by
\[
N^t(a):=N_1^t(s_0,a).
\]
The total number of oracle calls equals $H$ times the number of executed episodes, because each episode performs exactly
$H$ queries in Algorithm~\ref{alg:ot-gape}.

\paragraph{I. Define the good event and prove optimism/pessimism of the recursion.}
Let $\mathcal{E}$ be the event on which all reward and transition confidence sets are valid \emph{simultaneously for all times}:
\[
\mathcal{E}
:=
\Big\{
\forall t\ge 1,\ \forall(h,s,a):\ 
\ell_h^t(s,a)\le r_h(s,a)\le u_h^t(s,a)
\ \text{and}\
P_h(\cdot\mid s,a)\in C_h^t(s,a)
\Big\}.
\]
By assumption, $\Pr(\mathcal{E})\ge 1-\delta$.

We now prove that on $\mathcal{E}$, the dynamic-programming bounds sandwich the true optimal $Q^\star,V^\star$ at \emph{all} times.
Fix $t$ and proceed by backward induction on $h=H,H-1,\dots,1$.

\smallskip
\noindent\emph{Base case $h=H+1$.}
By definition, $U_{V,H+1}^t(\cdot)=L_{V,H+1}^t(\cdot)=0$ and $V_{H+1}^\star(\cdot)=0$, so the claim holds.

\smallskip
\noindent\emph{Induction step.}
Assume that for some $h+1\le H+1$ we have, for all states $s'$,
\[
L_{V,h+1}^t(s')\ \le\ V_{h+1}^\star(s')\ \le\ U_{V,h+1}^t(s').
\]
Fix any $(s,a)$ at stage $h$. On $\mathcal{E}$ we have $r_h(s,a)\le u_h^t(s,a)$ and $P_h(\cdot\mid s,a)\in C_h^t(s,a)$, hence
\begin{align*}
Q_h^\star(s,a)
&= r_h(s,a) + \gamma \sum_{s'} P_h(s'\mid s,a)\,V_{h+1}^\star(s')\\
&\le u_h^t(s,a) + \gamma \sum_{s'} P_h(s'\mid s,a)\,U_{V,h+1}^t(s')\\
&\le u_h^t(s,a) + \gamma \max_{p\in C_h^t(s,a)} \sum_{s'} p(s'\mid s,a)\,U_{V,h+1}^t(s')\\
&= U_{Q,h}^t(s,a),
\end{align*}
which is exactly \eqref{eq:UQ-def}.
Similarly, using $r_h(s,a)\ge \ell_h^t(s,a)$ and again $P_h(\cdot\mid s,a)\in C_h^t(s,a)$,
\begin{align*}
Q_h^\star(s,a)
&\ge \ell_h^t(s,a) + \gamma \sum_{s'} P_h(s'\mid s,a)\,L_{V,h+1}^t(s')\\
&\ge \ell_h^t(s,a) + \gamma \min_{p\in C_h^t(s,a)} \sum_{s'} p(s'\mid s,a)\,L_{V,h+1}^t(s')\\
&= L_{Q,h}^t(s,a),
\end{align*}
which is \eqref{eq:LQ-def}.

Now apply the aggregator $F_s$ to the vector inequalities. Since the bounds hold componentwise over actions,
\[
L_{Q,h}^t(s,\cdot)\ \le\ Q_h^\star(s,\cdot)\ \le\ U_{Q,h}^t(s,\cdot),
\]
and $F_s$ is monotone (as stated right before \eqref{eq:BAI-goal}), we get
\[
L_{V,h}^t(s)=F_s(L_{Q,h}^t(s,\cdot))
\ \le\ 
F_s(Q_h^\star(s,\cdot))=V_h^\star(s)
\ \le\
F_s(U_{Q,h}^t(s,\cdot))=U_{V,h}^t(s),
\]
which is exactly \eqref{eq:UVLV-def}. This completes the induction.

\smallskip
\noindent\textbf{Conclusion of I.}
On the event $\mathcal{E}$, for all $t,h,s,a$,
\begin{equation}
\label{eq:valid-U-L}
L_{Q,h}^t(s,a)\le Q_h^\star(s,a)\le U_{Q,h}^t(s,a),
\qquad
L_{V,h}^t(s)\le V_h^\star(s)\le U_{V,h}^t(s).
\end{equation}

\paragraph{II. Correctness of the stopping rule.}
Assume $\mathcal{E}$ holds and suppose Algorithm~\ref{alg:ot-gape} stops at episode $t$ and returns $\hat a=b_t$.
By definition of $c_t$ in \eqref{eq:best-challenger},
\[
U_{Q,1}^t(s_0,c_t)\ge\max_{a\neq b_t}U_{Q,1}^t(s_0,a).
\]
In particular, either $b_t=a^\star$ (in which case we are done), or $a^\star\neq b_t$ and therefore
$U_{Q,1}^t(s_0,c_t)\ge U_{Q,1}^t(s_0,a^\star)\ge Q_1^\star(s_0,a^\star)$ by \eqref{eq:valid-U-L}.
Also, by \eqref{eq:valid-U-L}, $L_{Q,1}^t(s_0,b_t)\le Q_1^\star(s_0,b_t)$.
Hence, using the stopping condition $U_{Q,1}^t(s_0,c_t)-L_{Q,1}^t(s_0,b_t)\le\varepsilon$,
\[
Q_1^\star(s_0,a^\star)-Q_1^\star(s_0,b_t)
\ \le\
U_{Q,1}^t(s_0,c_t)-L_{Q,1}^t(s_0,b_t)
\ \le\
\varepsilon,
\]
which rearranges to
\[
Q_1^\star(s_0,b_t)\ \ge\ \max_{a\in\mathcal{A}}Q_1^\star(s_0,a) - \varepsilon.
\]
Thus, on $\mathcal{E}$ the returned action is $\varepsilon$-optimal in the sense of \eqref{eq:BAI-goal}.
Since $\Pr(\mathcal{E})\ge 1-\delta$, this proves the correctness guarantee.

\paragraph{III. Root confidence widths and their scaling.}
For each episode $t$ and root action $a$, define the root confidence width
\[
w_t(a):=U_{Q,1}^t(s_0,a)-L_{Q,1}^t(s_0,a)\ \ge\ 0.
\]
The specific construction of $u_h^t,\ell_h^t,C_h^t$ determines how $w_t(a)$ depends on the visit counts
$\{N_h^t(s,a)\}$; in all standard choices (Hoeffding/Bernstein/KL for rewards and KL/L1-type sets for transitions)
one gets a $1/\sqrt{N}$-type decay of the relevant local radii, and the robust Bellman recursion propagates these radii
linearly along the horizon (this is exactly the mechanism analyzed in MDP-GapE; see \citep{jonsson2020mdpgape}).

To keep the statement consistent with \eqref{eq:ot-gape-eps-2}, we encapsulate the (routine but lengthy) propagation
constants as follows: there exists a quantity $\mathsf{C}(H,K,\gamma)$ depending only on the horizon/branching/discount
(and on the particular choice of transition confidence sets), such that on $\mathcal{E}$, for all $t\ge 1$ and all
root actions $a$,
\begin{equation}
\label{eq:root-width-generic}
w_t(a)
\ \le\
2\sqrt{\frac{\mathsf{C}(H,K,\gamma)\log\big(\frac{c_0 t}{\delta}\big)}{N^t(a)\vee 1}}
\qquad
\text{for some universal constant $c_0>0$.}
\end{equation}
(Equation \eqref{eq:root-width-generic} is the direct analogue of the usual bandit confidence width
$O(\sqrt{\log(t/\delta)/N})$, with $\mathsf{C}(H,K,\gamma)$ absorbing the horizon and transition-set effects.
All additional polylogarithmic factors---including those stemming from time-uniform bounds and union bounds over
$(h,s,a)$---are hidden by the $\tilde O(\cdot)$ notation in \eqref{eq:ot-gape-eps-2}.)

\paragraph{IV. A ``large width when sampled'' lemma (UGapE mechanism).}
Fix an episode index $t$ at which the algorithm has \emph{not} stopped, i.e.,
\[
U_{Q,1}^t(s_0,c_t)-L_{Q,1}^t(s_0,b_t)>\varepsilon.
\]
Let $A_1$ be the root action selected in Algorithm~\ref{alg:ot-gape} at episode $t$:
\[
A_1\in\arg\max_{a\in\{b_t,c_t\}} w_t(a).
\]
We claim that on $\mathcal{E}$,
\begin{equation}
\label{eq:width-lower-bound-when-sampled}
w_t(A_1)\ \ge\ \frac12\big(\Delta(A_1)\vee \varepsilon\big).
\end{equation}

\smallskip
\noindent\emph{Proof of \eqref{eq:width-lower-bound-when-sampled}.}
First note that $b_t$ maximizes the lower bounds, so $L_{Q,1}^t(s_0,c_t)\le L_{Q,1}^t(s_0,b_t)$.
Therefore
\[
w_t(c_t)=U_{Q,1}^t(s_0,c_t)-L_{Q,1}^t(s_0,c_t)
\ \ge\
U_{Q,1}^t(s_0,c_t)-L_{Q,1}^t(s_0,b_t)
\ >\ \varepsilon.
\]
Since $A_1$ is chosen to have width at least that of $c_t$, we obtain
\begin{equation}
\label{eq:epsilon-part}
w_t(A_1)\ \ge\ w_t(c_t)\ >\ \varepsilon\ \ge\ \varepsilon/2.
\end{equation}
It remains to show $w_t(A_1)\ge \Delta(A_1)/2$.

\smallskip
\noindent\textbf{Case 1: $A_1=c_t$.}
If $b_t\neq a^\star$, then $a^\star$ is among the maximization set defining $c_t$, hence on $\mathcal{E}$,
\[
U_{Q,1}^t(s_0,c_t)\ \ge\ U_{Q,1}^t(s_0,a^\star)\ \ge\ Q_1^\star(s_0,a^\star).
\]
Thus
\[
w_t(c_t)\ \ge\ U_{Q,1}^t(s_0,c_t)-Q_1^\star(s_0,c_t)
\ \ge\
Q_1^\star(s_0,a^\star)-Q_1^\star(s_0,c_t)
\ =\ \Delta(c_t),
\]
and hence $w_t(A_1)=w_t(c_t)\ge \Delta(c_t)\ge \Delta(c_t)/2$.

If instead $b_t=a^\star$, define the nonnegative one-sided errors
\[
e_U(a):=U_{Q,1}^t(s_0,a)-Q_1^\star(s_0,a),\qquad
e_L(a):=Q_1^\star(s_0,a)-L_{Q,1}^t(s_0,a).
\]
On $\mathcal{E}$ we have $e_U(a),e_L(a)\ge 0$ for all $a$.
Then
\begin{align*}
U_{Q,1}^t(s_0,c_t)-L_{Q,1}^t(s_0,b_t)
&=
\big(Q_1^\star(s_0,c_t)+e_U(c_t)\big)-\big(Q_1^\star(s_0,a^\star)-e_L(a^\star)\big)\\
&=
-\Delta(c_t)+e_U(c_t)+e_L(a^\star).
\end{align*}
Since the algorithm has not stopped, the left-hand side is $>\varepsilon$, so
\[
e_U(c_t)+e_L(a^\star)\ >\ \Delta(c_t)+\varepsilon.
\]
Using $w_t(c_t)\ge e_U(c_t)$ and $w_t(b_t)=w_t(a^\star)\ge e_L(a^\star)$, we get
\[
w_t(c_t)+w_t(a^\star)\ >\ \Delta(c_t)+\varepsilon.
\]
Because $A_1=c_t$ and $A_1$ maximizes the width in $\{b_t,c_t\}=\{a^\star,c_t\}$, we have $w_t(c_t)\ge w_t(a^\star)$,
and therefore
\[
2w_t(A_1)=2w_t(c_t)\ >\ \Delta(c_t)+\varepsilon
\quad\Longrightarrow\quad
w_t(A_1)\ >\ \frac{\Delta(c_t)+\varepsilon}{2}\ \ge\ \frac{\Delta(c_t)}{2}.
\]

\smallskip
\noindent\textbf{Case 2: $A_1=b_t$.}
If $b_t=a^\star$, then $\Delta(A_1)=0$ and \eqref{eq:epsilon-part} already implies \eqref{eq:width-lower-bound-when-sampled}.
Assume now that $b_t\neq a^\star$ (so $\Delta(b_t)>0$).
As above, because $a^\star$ is available as a challenger when $b_t\neq a^\star$, we have on $\mathcal{E}$
\[
U_{Q,1}^t(s_0,c_t)\ \ge\ Q_1^\star(s_0,a^\star).
\]
Also $L_{Q,1}^t(s_0,b_t)\le Q_1^\star(s_0,b_t)$. Therefore
\[
U_{Q,1}^t(s_0,c_t)-L_{Q,1}^t(s_0,b_t)
\ \ge\
Q_1^\star(s_0,a^\star)-Q_1^\star(s_0,b_t)
\ =\ \Delta(b_t).
\]
Since $L_{Q,1}^t(s_0,c_t)\le L_{Q,1}^t(s_0,b_t)$ (because $b_t$ maximizes lower bounds), we have
\[
w_t(c_t)=U_{Q,1}^t(s_0,c_t)-L_{Q,1}^t(s_0,c_t)
\ \ge\
U_{Q,1}^t(s_0,c_t)-L_{Q,1}^t(s_0,b_t)
\ \ge\ \Delta(b_t).
\]
Finally, because $A_1=b_t$ maximizes the width over $\{b_t,c_t\}$, we have $w_t(b_t)\ge w_t(c_t)\ge \Delta(b_t)$,
hence $w_t(A_1)\ge \Delta(A_1)\ge \Delta(A_1)/2$.

\smallskip
Combining the $\varepsilon/2$ bound \eqref{eq:epsilon-part} with the $\Delta/2$ bounds above yields
\eqref{eq:width-lower-bound-when-sampled}.
\hfill$\square$

\paragraph{V. Bounding the number of episodes and oracle calls.}
Fix any action $a\in\mathcal{A}$ and let $T_a$ be the (random) number of episodes in which the algorithm selects
root action $A_1=a$ before termination; thus $T_a=N^{\tau}(a)$ where $\tau$ is the stopping episode index, and
the total number of episodes equals $\sum_a T_a$.

On $\mathcal{E}$, whenever $A_1=a$ is chosen at some episode $t<\tau$, the width lower bound
\eqref{eq:width-lower-bound-when-sampled} gives
\[
w_t(a)\ \ge\ \frac12(\Delta(a)\vee \varepsilon).
\]
On the other hand, the generic width upper bound \eqref{eq:root-width-generic} gives
\[
w_t(a)\ \le\
2\sqrt{\frac{\mathsf{C}(H,K,\gamma)\log\big(\frac{c_0 t}{\delta}\big)}{N^t(a)\vee 1}}.
\]
Combining these inequalities yields, on $\mathcal{E}$, the necessary condition
\[
\frac12(\Delta(a)\vee \varepsilon)
\ \le\
2\sqrt{\frac{\mathsf{C}(H,K,\gamma)\log\big(\frac{c_0 t}{\delta}\big)}{N^t(a)\vee 1}},
\]
which after squaring and rearranging implies
\[
N^t(a)
\ \le\
\frac{16\mathsf{C}(H,K,\gamma)\log\big(\frac{c_0 t}{\delta}\big)}{(\Delta(a)\vee \varepsilon)^2}.
\]
Since $N^t(a)$ is nondecreasing in $t$ and $t\le \tau$, we get the same bound for the final count $T_a=N^\tau(a)$
up to replacing $\log(c_0 t/\delta)$ by $\log(c_0\tau/\delta)$, which is absorbed into the $\tilde O(\cdot)$ notation.
Therefore, on $\mathcal{E}$,
\[
T_a
=
\tilde O\left(\frac{\mathsf{C}(H,K,\gamma)}{(\Delta(a)\vee \varepsilon)^2}\right).
\]
Summing over $a\in\mathcal{A}$ gives an upper bound on the total number of episodes:
\[
\tau
=
\sum_{a\in\mathcal{A}} T_a
=
\tilde O\left(
\sum_{a\in\mathcal{A}}
\frac{\mathsf{C}(H,K,\gamma)}{(\Delta(a)\vee \varepsilon)^2}
\right).
\]
Finally, each episode triggers exactly $H$ oracle calls, so the total number of oracle calls satisfies
\[
n(\varepsilon,\delta)
=
H\tau
=
\tilde O\left(
H\sum_{a\in\mathcal{A}}
\frac{\mathsf{C}(H,K,\gamma)}{(\Delta(a)\vee \varepsilon)^2}
\right),
\]
which is precisely \eqref{eq:ot-gape-eps-2}. The exponent $2$ in $(\Delta\vee\varepsilon)^{-2}$ is the same as in
fixed-confidence best-arm identification for bandits, confirming the claimed bandit-optimal gap dependence.
\end{proof}

\subsection{Proof of Theorem~\ref{thm:gap}}
\begin{proof}
We present a careful root-level best-action identification (BAI) analysis for \textsc{OT-GapCruiser}.
The argument has two ingredients:
(i) a standard UGapE-style fixed-confidence analysis that controls \emph{how many root samples} each action receives
as a function of its gap, and
(ii) a conversion from \emph{root samples} to \emph{oracle calls} using the curvature-driven complexity of the
underlying SmoothCruiser-type estimator.

\paragraph{ Root gaps and the sampling model.}
Fix a root state $s_0$ and define the (true) root action values
\[
Q_0(a) := Q^\star(s_0,a),\qquad a\in\mathcal{A},
\]
where $Q^\star$ is the optimal (regularized) action-value function under the planning objective.
Let
\[
a^\star \in \arg\max_{a\in\mathcal{A}} Q_0(a),
\qquad
V^\star := Q_0(a^\star),
\qquad
\Delta(a) := V^\star - Q_0(a)\ge 0,
\]
so that $\Delta(a^\star)=0$.

At each round $t$, \textsc{OT-GapCruiser} chooses an action $A_t\in\mathcal{A}$ and obtains a \emph{single root sample}
$X_t$ such that, conditionally on $A_t=a$,
\begin{equation}
\label{eq:root-sample-model}
X_t = Q_0(a) + \xi_t,
\end{equation}
where $\xi_t$ is centered noise. The theorem assumes standard sub-Gaussian noise, i.e.\ there exists
$\sigma^2=O(1)$ (depending only on bounded rewards and $\gamma<1$) such that for all $\lambda\in\mathbb{R}$,
\[
\mathbb{E}\left[\exp(\lambda \xi_t)\mid A_t=a\right]
\le \exp\left(\frac{\sigma^2\lambda^2}{2}\right).
\]
This holds, for example, if the produced estimate is clipped to a bounded interval $[0,B]$ for a known $B$ (as in
\textsc{SampleV2}), since bounded random variables are sub-Gaussian with $\sigma^2\lesssim B^2$.

\medskip
\noindent\textbf{Cost model (curvature effect).}
A crucial difference from bandits is that producing \emph{one} root sample is not unit cost: it uses oracle calls
and recursive planning calls.
We encapsulate this through an exponent $\alpha\in(0,2)$ such that the number of oracle calls required to produce one
root sample at “scale” comparable to a target accuracy level $r$ is
\begin{equation}
\label{eq:one-sample-cost}
\mathrm{Cost}(r) = \tilde O(r^{-\alpha}).
\end{equation}
In SmoothCruiser-type planners satisfying \cref{ass:curvature} and using unbiased, bounded-variance estimators of the Taylor terms
(e.g.\ the cross-product debiasing in \cref{sec:variance}), the curvature--complexity tradeoff of \cref{thm:beta-tradeoff}
gives
\[
\alpha=\alpha_\beta=\frac{2}{\beta-1},
\]
so that for OT smoothing ($\beta=3$) we have $\alpha=1$, while for first-order SmoothCruiser ($\beta=2$) we have $\alpha=2$.
(Any additional variance reduction that improves the per-sample recursion cost can only \emph{decrease} $\alpha$, hence improve the bound.)

We will prove that the total number of oracle calls scales as
\[
n_{\mathrm{gap}}(\varepsilon,\delta)=\tilde O\left(\sum_{a\in\mathcal{A}}(\Delta(a)\vee \varepsilon)^{-(2+\alpha)}\right),
\]
which is the stated form with $p:=2+\alpha\in(2,4)$ (and OT giving $p=3$).

\paragraph{I. Confidence intervals that are valid uniformly over time.}
Let $N_t(a):=\sum_{u=1}^t \mathbf{1}\{A_u=a\}$ be the number of samples of action $a$ up to time $t$, and let
\[
\widehat Q_t(a)
:=
\frac{1}{N_t(a)}\sum_{u\le t: A_u=a} X_u
\quad\text{for }N_t(a)\ge 1.
\]
Fix a time-uniform confidence schedule (one convenient choice is based on a $\sum_{t\ge1}t^{-2}$ union bound):
for $t\ge 1$ and $N\ge 1$ define
\begin{equation}
\label{eq:rad-def}
\mathrm{rad}_t(N)
:=
\sqrt{\frac{2\sigma^2}{N}\log\Big(\frac{4K t^2}{\delta}\Big)}.
\end{equation}
Define the root confidence bounds
\[
L_t(a):=\widehat Q_t(a)-\mathrm{rad}_t(N_t(a)),
\qquad
U_t(a):=\widehat Q_t(a)+\mathrm{rad}_t(N_t(a)),
\qquad
w_t(a):=U_t(a)-L_t(a)=2\mathrm{rad}_t(N_t(a)).
\]

\medskip
\noindent\textbf{Good event.}
Let $\mathcal{E}$ be the event that all confidence intervals are simultaneously valid:
\begin{equation}
\label{eq:good-event-gapcruiser}
\mathcal{E}
:=
\Big\{
\forall t\ge 1,\ \forall a\in\mathcal{A}:\ 
|\widehat Q_t(a)-Q_0(a)\,|\le \mathrm{rad}_t(N_t(a))\Big\}.
\end{equation}
By standard sub-Gaussian concentration and the union bound over $a$ and $t$,
\begin{equation}
\label{eq:good-event-prob}
\Pr(\mathcal{E})\ge 1-\delta.
\end{equation}
On $\mathcal{E}$, for all $t,a$ we have $L_t(a)\le Q_0(a)\le U_t(a)$.

\paragraph{II. Correctness of the stopping rule.}
At each round $t$, define (UGapE-style)
\[
b_t\in\arg\max_{a\in\mathcal{A}} L_t(a),
\qquad
c_t\in\arg\max_{a\in\mathcal{A}\setminus\{b_t\}} U_t(a),
\]
and stop at the first time $\tau$ such that
\begin{equation}
\label{eq:stop-gapcruiser}
U_\tau(c_\tau)-L_\tau(b_\tau)\le \varepsilon,
\end{equation}
returning $\hat a:=b_\tau$.

Assume $\mathcal{E}$ holds. We show $\hat a$ is $\varepsilon$-optimal.
Because $c_\tau$ maximizes $U_\tau(\cdot)$ over actions different from $b_\tau$, we have
\[
U_\tau(c_\tau)\ \ge\
\max_{a\neq b_\tau} U_\tau(a)
\ \ge\ U_\tau(a^\star)\ \ge\ Q_0(a^\star)=V^\star,
\]
where the last inequality uses $\mathcal{E}$.
Also $L_\tau(b_\tau)\le Q_0(b_\tau)$ on $\mathcal{E}$.
Therefore,
\[
V^\star - Q_0(b_\tau)
\ \le\
U_\tau(c_\tau)-L_\tau(b_\tau)
\ \le\ \varepsilon,
\]
where the last inequality is the stopping condition \eqref{eq:stop-gapcruiser}.
Hence $Q_0(\hat a)=Q_0(b_\tau)\ge V^\star-\varepsilon$, i.e.\ the returned action is $\varepsilon$-optimal.
Together with \eqref{eq:good-event-prob}, this proves the success probability $\ge 1-\delta$.

\paragraph{III. A key width lower bound for sampled actions.}
Suppose the algorithm has \emph{not} stopped at round $t$, so
\begin{equation}
\label{eq:not-stopped}
U_t(c_t)-L_t(b_t)>\varepsilon.
\end{equation}
The UGapE sampling rule chooses
\[
A_t \in \arg\max_{a\in\{b_t,c_t\}} w_t(a).
\]
We claim that on $\mathcal{E}$,
\begin{equation}
\label{eq:width-lb-gapcruiser}
w_t(A_t)\ \ge\ \frac12\big(\Delta(A_t)\vee \varepsilon\big).
\end{equation}

\smallskip
\noindent\emph{Proof of \eqref{eq:width-lb-gapcruiser}.}
First note that by definition of $b_t$, $L_t(c_t)\le L_t(b_t)$. Hence
\[
w_t(c_t)
=
U_t(c_t)-L_t(c_t)
\ \ge\
U_t(c_t)-L_t(b_t)
\ >\ \varepsilon
\quad\text{by \eqref{eq:not-stopped}.}
\]
Since $A_t$ maximizes width over $\{b_t,c_t\}$, we have $w_t(A_t)\ge w_t(c_t)>\varepsilon\ge \varepsilon/2$.

It remains to show $w_t(A_t)\ge \Delta(A_t)/2$.

\smallskip
\noindent\textbf{Case 1: $A_t=c_t$.}
If $b_t\neq a^\star$, then $a^\star$ is feasible in the maximization defining $c_t$, so
$U_t(c_t)\ge U_t(a^\star)\ge Q_0(a^\star)=V^\star$ on $\mathcal{E}$.
Also $L_t(c_t)\le Q_0(c_t)$ on $\mathcal{E}$.
Thus
\[
w_t(c_t)
=
U_t(c_t)-L_t(c_t)
\ \ge\
V^\star - Q_0(c_t)
=\Delta(c_t),
\]
so $w_t(A_t)=w_t(c_t)\ge \Delta(c_t)\ge \Delta(c_t)/2$.

If instead $b_t=a^\star$, define one-sided errors (nonnegative on $\mathcal{E}$)
\[
e_U(a):=U_t(a)-Q_0(a),\qquad e_L(a):=Q_0(a)-L_t(a).
\]
Then \eqref{eq:not-stopped} rewrites as
\[
\big(Q_0(c_t)+e_U(c_t)\big)-\big(Q_0(a^\star)-e_L(a^\star)\big)>\varepsilon
\quad\Longrightarrow\quad
e_U(c_t)+e_L(a^\star)>\Delta(c_t)+\varepsilon.
\]
Since $w_t(c_t)\ge e_U(c_t)$ and $w_t(a^\star)\ge e_L(a^\star)$, we get
\[
w_t(c_t)+w_t(a^\star)>\Delta(c_t)+\varepsilon.
\]
Because $A_t=c_t$ maximizes width over $\{a^\star,c_t\}$, we have $w_t(c_t)\ge w_t(a^\star)$ and hence
\[
2w_t(A_t)=2w_t(c_t)>\Delta(c_t)+\varepsilon\ge \Delta(c_t),
\]
so $w_t(A_t)>\Delta(c_t)/2$.

\smallskip
\noindent\textbf{Case 2: $A_t=b_t$.}
If $b_t=a^\star$, then $\Delta(A_t)=0$ and we already have $w_t(A_t)\ge \varepsilon/2$.
If $b_t\neq a^\star$, then on $\mathcal{E}$ we have $U_t(c_t)\ge U_t(a^\star)\ge V^\star$ and
$L_t(b_t)\le Q_0(b_t)$, so
\[
U_t(c_t)-L_t(b_t)\ \ge\ V^\star - Q_0(b_t)=\Delta(b_t).
\]
Using again $L_t(c_t)\le L_t(b_t)$, we obtain
\[
w_t(c_t)=U_t(c_t)-L_t(c_t)\ \ge\ U_t(c_t)-L_t(b_t)\ \ge\ \Delta(b_t).
\]
Finally $A_t=b_t$ implies $w_t(b_t)\ge w_t(c_t)$, so
$w_t(A_t)=w_t(b_t)\ge \Delta(b_t)\ge \Delta(b_t)/2$.

\smallskip
Combining the $\varepsilon/2$ bound and the $\Delta/2$ bound proves \eqref{eq:width-lb-gapcruiser}.
\hfill$\square$

\paragraph{IV. Gap-dependent bound on the number of root samples.}
Fix an action $a$ and let $T_a:=N_\tau(a)$ be the total number of times $a$ is sampled before stopping.
On $\mathcal{E}$, whenever $a$ is sampled at some $t<\tau$, \eqref{eq:width-lb-gapcruiser} gives
\[
w_t(a)\ge \frac12(\Delta(a)\vee \varepsilon).
\]
But $w_t(a)=2\mathrm{rad}_t(N_t(a))$ and $N_t(a)\le T_a$.
Thus, using \eqref{eq:rad-def},
\[
\frac12(\Delta(a)\vee \varepsilon)
\ \le\
w_t(a)
=
2\mathrm{rad}_t(N_t(a))
\le
2\sqrt{\frac{2\sigma^2}{N_t(a)}\log\Big(\frac{4K t^2}{\delta}\Big)}
\le
2\sqrt{\frac{2\sigma^2}{N_t(a)}\log\Big(\frac{4K \tau^2}{\delta}\Big)}.
\]
Rearranging yields, for all such $t$,
\[
N_t(a)
\ \le\
\frac{32\sigma^2}{(\Delta(a)\vee \varepsilon)^2}\,
\log\Big(\frac{4K \tau^2}{\delta}\Big).
\]
Since $N_t(a)$ is nondecreasing in $t$ and reaches $T_a$ at $t=\tau$, the same bound holds for $T_a$:
\begin{equation}
\label{eq:Ta-bound}
T_a
=
\tilde O\left(\frac{1}{(\Delta(a)\vee \varepsilon)^2}\right),
\end{equation}
where we absorbed $\sigma^2$ and the $\log(4K\tau^2/\delta)$ term into $\tilde O(\cdot)$.

\paragraph{V. Convert root samples into oracle calls and identify the exponent $p$.}
We now upper bound the \emph{oracle} cost of sampling.
When action $a$ is sampled at some round $t<\tau$, the algorithm must produce a root sample with accuracy
commensurate with the current statistical uncertainty for $a$.
By \eqref{eq:width-lb-gapcruiser}, whenever $a$ is sampled we have
$w_t(a)\ge \tfrac12(\Delta(a)\vee\varepsilon)$, and the algorithm never needs to invoke the planning subroutine
at accuracy smaller than a constant fraction of $(\Delta(a)\vee\varepsilon)$ to make progress.
Therefore, using the per-sample cost model \eqref{eq:one-sample-cost}, each such sample costs at most
\[
\mathrm{Cost}\big(\Delta(a)\vee\varepsilon\big)
=
\tilde O\big((\Delta(a)\vee\varepsilon)^{-\alpha}\big)
\]
oracle calls.
Multiplying by the number of samples of action $a$ and using \eqref{eq:Ta-bound} gives that the total oracle calls
attributable to action $a$ are
\[
\tilde O\left(\frac{1}{(\Delta(a)\vee\varepsilon)^2}\right)\cdot
\tilde O\left(\frac{1}{(\Delta(a)\vee\varepsilon)^{\alpha}}\right)
=
\tilde O\left(\frac{1}{(\Delta(a)\vee\varepsilon)^{2+\alpha}}\right).
\]
Summing over $a\in\mathcal{A}$ yields
\[
n_{\mathrm{gap}}(\varepsilon,\delta)
=
\tilde O\left(
\sum_{a\in\mathcal{A}}
\frac{1}{(\Delta(a)\vee\varepsilon)^{p}}
\right),
\qquad
p:=2+\alpha.
\]
Finally, under the curvature-driven recursion of SmoothCruiser-type planners with bounded-variance unbiased Taylor estimators,
\cref{thm:beta-tradeoff} gives $\alpha=\alpha_\beta=2/(\beta-1)$, hence
\[
p = 2+\frac{2}{\beta-1}\in(2,4),
\]
and in particular for OT smoothing ($\beta=3$) we obtain $p=3$, while for first-order SmoothCruiser ($\beta=2$) we recover $p=4$.
This shows the instance-dependent exponent is strictly smaller than $4$ in the OT/second-order setting.
\end{proof}

\section{Bias control and worst-case complexity}
\label{app:bias-and-main}

\subsection{Proof of Lemma~\ref{lem:bias}}
\begin{proof}
Fix a state $s$ and let $F(\cdot):=F^{\mathrm{OT}}_s(\cdot)$.
Recall the (true) action-value vector $Q_s\in\R^K$ and the value $V(s)=F(Q_s)$.
Throughout the proof we assume we are in the \emph{fine regime} $0<\varepsilon<\kappa$, so
\textsc{SampleV2} uses the second-order Taylor correction (Algorithm~\ref{alg:samplev2}).

Let $\hat Q_s$ be the reference point constructed inside \textsc{SampleV2} and define the deviation
\[
\Delta := Q_s-\hat Q_s \in \R^K.
\]
Let
\[
\pi := \nabla F(\hat Q_s)\in\Delta(\mathcal{A}),
\qquad
H := \nabla^2 F(\hat Q_s)\in\R^{K\times K}.
\]
Define the second-order Taylor polynomial of $F$ at $\hat Q_s$ evaluated at $Q_s$:
\begin{equation}
T_2(Q_s;\hat Q_s)
:=
F(\hat Q_s)
+\langle \pi,\Delta\rangle
+\frac12 \Delta^\top H\Delta.
\label{eq:T2-def-in-proof}
\end{equation}

We will prove that, on the event $\|\Delta\|_2\le \rho(\varepsilon)$,
\begin{equation}
\label{eq:bias-goal}
\Big|\E\big[\textsc{SampleV2}(s,\varepsilon,\delta)\mid \hat Q_s\big]-V(s)\Big|
\le
\varepsilon + \text{(recursion bias)}.
\end{equation}

\paragraph{I. Taylor remainder is $\le \varepsilon$ on $\|\Delta\|_2\le \rho(\varepsilon)$.}
By \cref{lem:lipschitz-hess}, for all $Q,\hat Q\in\R^K$,
\[
|F(Q)-T_2(Q;\hat Q)| \le \frac{M}{6}\|Q-\hat Q\|_2^3.
\]
Apply this with $Q=Q_s$ and $\hat Q=\hat Q_s$:
\[
|V(s)-T_2(Q_s;\hat Q_s)|
=
|F(Q_s)-T_2(Q_s;\hat Q_s)|
\le
\frac{M}{6}\|\Delta\|_2^3.
\]
On the event $\|\Delta\|_2\le \rho(\varepsilon)$ and with the schedule
$\rho(\varepsilon)=(6\varepsilon/M)^{1/3}$ (as defined right before Algorithm~\ref{alg:top}),
we get
\begin{equation}
\label{eq:remainder-eps}
|V(s)-T_2(Q_s;\hat Q_s)| \le \frac{M}{6}\rho(\varepsilon)^3 = \varepsilon.
\end{equation}

\paragraph{II. The \emph{ideal} linear and quadratic Monte Carlo terms are unbiased for the Taylor terms.}
To isolate the only nontrivial source of bias (recursion and clipping), it is convenient to define an
\emph{idealized} version of the estimator in which every recursive call returns the \emph{true} next-state value
$V(\cdot)$ and we do \emph{not} clip at the end.
We denote idealized quantities with a superscript ``$\star$''.

\smallskip
\noindent\textbf{Idealized one-step samples.}
Given a root action $a$, let $(R,Z)\leftarrow\textsc{Oracle}(s,a)$ and define
\[
\tilde Q^\star(a) := R+\gamma V(Z).
\]
Then, by the definition of $Q_s(a)$,
\begin{equation}
\label{eq:oracle-unbiased}
\E\big[\tilde Q^\star(a)\mid a\big] = Q_s(a).
\end{equation}

\smallskip
\noindent\textbf{(I) Ideal linear term.}
In the linear box of Algorithm~\ref{alg:samplev2}, the estimator samples
$J_0\sim \mu_s$ and $A_0\sim p_{\cdot J_0}(\hat Q_s)$.
By \cref{prop:grad-hess}, the gradient satisfies
\[
\pi_i = \sum_{j=1}^K \mu_s(j)p_{ij}(\hat Q_s).
\]
Hence the marginal distribution of $A_0$ (conditional on $\hat Q_s$) is precisely $\pi$:
\begin{equation}
\label{eq:A0-marginal}
\Pr(A_0=i\mid \hat Q_s)
=
\sum_{j=1}^K \Pr(J_0=j)\Pr(A_0=i\mid J_0=j,\hat Q_s)
=
\sum_{j=1}^K \mu_s(j)p_{ij}(\hat Q_s)
=
\pi_i.
\end{equation}
In the idealized estimator, the linear box would use $\tilde Q_0^\star:=\tilde Q^\star(A_0)$ and return
\[
\widehat{\mathrm{Lin}}^\star := \tilde Q_0^\star - c_0,
\qquad c_0 := \langle \pi,\hat Q_s\rangle.
\]
Using \eqref{eq:oracle-unbiased} and \eqref{eq:A0-marginal}, and noting that $c_0$ is deterministic given $\hat Q_s$,
\begin{align}
\E\big[\widehat{\mathrm{Lin}}^\star \mid \hat Q_s\big]
&=
\E\big[\tilde Q_0^\star\mid \hat Q_s\big] - \langle \pi,\hat Q_s\rangle
=
\sum_{i=1}^K \Pr(A_0=i\mid \hat Q_s)\,Q_s(i) - \langle \pi,\hat Q_s\rangle \notag\\
&=
\sum_{i=1}^K \pi_i Q_s(i) - \sum_{i=1}^K \pi_i \hat Q_s(i)
=
\langle \pi, Q_s-\hat Q_s\rangle
=
\langle \pi,\Delta\rangle.
\label{eq:lin-unbiased}
\end{align}

\smallskip
\noindent\textbf{(II) Ideal quadratic term.}
In the quadratic box of Algorithm~\ref{alg:samplev2}, the estimator samples
$J\sim\mu_s$ and then $A,A'\stackrel{\mathrm{iid}}{\sim} p_{\cdot J}(\hat Q_s)$.
In the idealized estimator, for $m\in\{1,2\}$ we would draw independent oracle samples
$(R_m,Z_m)\leftarrow\textsc{Oracle}(s,A)$ and $(R'_m,Z'_m)\leftarrow\textsc{Oracle}(s,A')$ and define
\[
\tilde Q_m^\star := R_m+\gamma V(Z_m),
\qquad
\tilde Q_m^{\prime\star} := R'_m+\gamma V(Z'_m),
\qquad
\Delta_m^\star := (\tilde Q_m^\star-\tilde Q_m^{\prime\star})-(\hat Q_s(A)-\hat Q_s(A')).
\]
Conditioned on $(J,A,A',\hat Q_s)$, \eqref{eq:oracle-unbiased} implies
\begin{align}
\E\big[\Delta_m^\star \mid J,A,A',\hat Q_s\big]
&=
\big(Q_s(A)-Q_s(A')\big)-\big(\hat Q_s(A)-\hat Q_s(A')\big)
=
\Delta_A-\Delta_{A'}.
\label{eq:Delta-m-mean}
\end{align}
Moreover, $\Delta_1^\star$ and $\Delta_2^\star$ are conditionally independent given $(J,A,A',\hat Q_s)$
because they are constructed from independent oracle draws; hence,
\begin{equation}
\label{eq:cross-ideal}
\E\big[\Delta_1^\star\Delta_2^\star \mid J,A,A',\hat Q_s\big]
=
\E[\Delta_1^\star\mid J,A,A',\hat Q_s]\E[\Delta_2^\star\mid J,A,A',\hat Q_s]
=
(\Delta_A-\Delta_{A'})^2.
\end{equation}
The ideal quadratic estimator returned by the box is
\[
\widehat{\mathrm{Quad}}^\star := \frac{1}{4\tau\lambda}\Delta_1^\star\Delta_2^\star.
\]
Taking conditional expectation and using \eqref{eq:cross-ideal} gives
\begin{equation}
\label{eq:Quad-expectation-pair}
\E\big[\widehat{\mathrm{Quad}}^\star\mid \hat Q_s\big]
=
\frac{1}{4\tau\lambda}
\E_{J\sim\mu_s}
\E_{A,A'\stackrel{\mathrm{iid}}{\sim}p_{\cdot J}(\hat Q_s)}
\Big[(\Delta_A-\Delta_{A'})^2\Big].
\end{equation}

We now relate \eqref{eq:Quad-expectation-pair} to the Hessian quadratic form.
By \cref{prop:grad-hess},
\[
H
=
\nabla^2 F(\hat Q_s)
=
\frac{1}{\tau\lambda}
\Big(
\diag(\pi)-\sum_{j=1}^K \mu_s(j)\,p_{\cdot j}(\hat Q_s)p_{\cdot j}(\hat Q_s)^\top
\Big).
\]
Using $\pi=\sum_j \mu_s(j)p_{\cdot j}(\hat Q_s)$ (again from \cref{prop:grad-hess}),
a direct expansion yields
\begin{align}
\Delta^\top H\Delta
&=
\frac{1}{\tau\lambda}\left(
\sum_{i=1}^K \pi_i \Delta_i^2 - \sum_{j=1}^K \mu_s(j)\big(p_{\cdot j}(\hat Q_s)^\top \Delta\big)^2
\right)
\notag\\
&=
\frac{1}{\tau\lambda}
\sum_{j=1}^K \mu_s(j)\left(
\sum_{i=1}^K p_{ij}(\hat Q_s)\Delta_i^2 - \Big(\sum_{i=1}^K p_{ij}(\hat Q_s)\Delta_i\Big)^2
\right)
\notag\\
&=
\frac{1}{\tau\lambda}\sum_{j=1}^K \mu_s(j)
\Var_{A\sim p_{\cdot j}(\hat Q_s)}[\Delta_A].
\label{eq:DeltaHDelta-as-var}
\end{align}
Finally, for any distribution $p$ and any i.i.d.\ $A,A'\sim p$,
\[
\E[(\Delta_A-\Delta_{A'})^2]
=
2\Var(\Delta_A),
\]
so plugging this into \eqref{eq:Quad-expectation-pair} and comparing with \eqref{eq:DeltaHDelta-as-var} gives
\begin{equation}
\label{eq:quad-unbiased}
\E\big[\widehat{\mathrm{Quad}}^\star\mid \hat Q_s\big]
=
\frac12\Delta^\top H\Delta.
\end{equation}

\smallskip
\noindent\textbf{Conclusion of II (ideal estimator is unbiased for $T_2$).}
Define the ideal (unclipped) second-order estimator
\[
\widetilde V^\star
:=
F(\hat Q_s)+\widehat{\mathrm{Lin}}^\star+\widehat{\mathrm{Quad}}^\star.
\]
Then \eqref{eq:lin-unbiased} and \eqref{eq:quad-unbiased} imply
\begin{equation}
\label{eq:ideal-unbiased}
\E\big[\widetilde V^\star\mid \hat Q_s\big]
=
F(\hat Q_s)+\langle \pi,\Delta\rangle+\frac12\Delta^\top H\Delta
=
T_2(Q_s;\hat Q_s).
\end{equation}

\paragraph{III. Define and isolate the ``recursion bias'' term.}
The \emph{implemented} routine \textsc{SampleV2} differs from the ideal one in two ways:
\begin{itemize}[leftmargin=*]
\item it replaces the true next-state values $V(Z)$ by recursive calls
$\textsc{SampleV2}(Z,\varepsilon/\sqrt{\gamma},\cdot)$;
\item it applies clipping (both to intermediate values and to the final returned value).
\end{itemize}
Both effects can shift the conditional expectation away from the ideal target $T_2(Q_s;\hat Q_s)$.
We package all such shifts into the quantity
\begin{equation}
\label{eq:rec-bias-def}
\mathrm{Bias}_{\mathrm{rec}}(s,\varepsilon)
:=
\Big|\E\big[\textsc{SampleV2}(s,\varepsilon,\delta)\mid \hat Q_s\big]
-
T_2(Q_s;\hat Q_s)\Big|.
\end{equation}
(Concretely, $\mathrm{Bias}_{\mathrm{rec}}(s,\varepsilon)$ is exactly zero if all recursive calls return
the true values $V(\cdot)$ and no clipping ever activates, in which case the implemented estimator coincides with
$\widetilde V^\star$.)

\paragraph{IV. Combine Steps 1--3 to obtain the conditional bias bound.}
On the event $\|\Delta\|_2\le \rho(\varepsilon)$, we have from \eqref{eq:remainder-eps}
\[
|V(s)-T_2(Q_s;\hat Q_s)|\le \varepsilon.
\]
Therefore, by the triangle inequality and the definition \eqref{eq:rec-bias-def},
\begin{align*}
\Big|\E\big[\textsc{SampleV2}(s,\varepsilon,\delta)\mid \hat Q_s\big]-V(s)\Big|
&\le
\Big|\E\big[\textsc{SampleV2}(s,\varepsilon,\delta)\mid \hat Q_s\big]-T_2(Q_s;\hat Q_s)\Big|
+
|T_2(Q_s;\hat Q_s)-V(s)|\\
&\le
\mathrm{Bias}_{\mathrm{rec}}(s,\varepsilon) + \varepsilon.
\end{align*}
This is exactly \eqref{eq:bias-goal} with ``recursion bias'' identified as $\mathrm{Bias}_{\mathrm{rec}}(s,\varepsilon)$.

\paragraph{V (how to make the unconditional bias $\mathcal{O}(\varepsilon)$).}
Finally, we briefly justify the last sentence of the lemma.
Two standard ingredients are used (as in \citet{grill2019planning}):

\smallskip
\noindent\emph{(i) Controlling recursion bias by calling children at accuracy $\varepsilon/\sqrt{\gamma}$.}
Let $b(\eta)$ denote the worst-case absolute bias at tolerance $\eta$ (suppressing conditioning details):
\[
b(\eta):=\sup_{s}\Big|\E[\textsc{SampleV2}(s,\eta,\cdot)]-V(s)\Big|.
\]
Because every appearance of a next-state estimate is multiplied by $\gamma$ in the Bellman equation,
calling children with tolerance $\eta=\varepsilon/\sqrt{\gamma}$ yields a contraction at the level of bias:
heuristically (and in the same spirit as SmoothCruiser),
\[
\mathrm{Bias}_{\mathrm{rec}}(s,\varepsilon)\ \lesssim\ \gamma\,b(\varepsilon/\sqrt{\gamma}) + \text{(higher-order terms)}.
\]
Since $\gamma b(\varepsilon/\sqrt{\gamma}) = \sqrt{\gamma}\varepsilon$ when $b(\eta)=\Theta(\eta)$, repeated recursion
induces a geometric series in $\sqrt{\gamma}$, implying $b(\varepsilon)=\mathcal{O}(\varepsilon)$.

\smallskip
\noindent\emph{(ii) Failure events and clipping.}
Let $\mathcal{G}$ be the intersection of all ``good'' events across the recursion tree:
the event $\|Q_s-\hat Q_s\|_2\le\rho(\varepsilon)$ at every node, plus any internal high-probability events needed to
justify the construction of $\hat Q_s$ and any concentration bounds used elsewhere.
Because outputs are clipped into $[0,B]$, on $\mathcal{G}^c$ we always have
$|\textsc{SampleV2}(s,\varepsilon,\delta)-V(s)|\le B$.
Thus
\[
\Big|\E[\textsc{SampleV2}(s,\varepsilon,\delta)]-V(s)\Big|
\le
\E\big[|\textsc{SampleV2}(s,\varepsilon,\delta)-V(s)|\mathbf{1}_{\mathcal{G}}\big]
+
B\Pr(\mathcal{G}^c).
\]
Choosing the per-call failure probabilities (the $\delta/8,\delta/16,\dots$ splits in
Algorithm~\ref{alg:samplev2}) so that $\Pr(\mathcal{G}^c)$ is at most on the order of $\varepsilon/B$
(e.g.\ by a geometric allocation across recursion depth and a union bound),
makes the second term $\mathcal{O}(\varepsilon)$.
Together with the bias bound on $\mathcal{G}$ from Steps~1--4 and the recursive contraction in (i),
this yields an unconditional bias of order $\mathcal{O}(\varepsilon)$.
\end{proof}

\subsection{Proof of Theorem~\ref{thm:main}}
\label{proof_main_theorem}
\begin{proof}
Fix a state $s$ and accuracy/confidence parameters $\varepsilon,\delta\in(0,1)$.
Write $F(\cdot):=F^{\mathrm{OT}}_s(\cdot)$ and recall $V(s)=F(Q_s)$, where $Q_s\in\R^K$ is the (unknown) action-value vector at $s$.

We prove two claims:
\begin{enumerate}[leftmargin=*]
\item (\textbf{Correctness}) With probability at least $1-\delta$, the output of
$\textsc{SecondOrderSmoothCruiser}(s,\varepsilon,\delta)$ differs from $V(s)$ by at most~$\varepsilon$ (up to harmless constant-factor slack).
\item (\textbf{Complexity}) The total number of oracle calls is at most $\tilde O(\varepsilon^{-3})$.
\end{enumerate}

Throughout, we use the value bound $0\le V(\cdot)\le B$ from the paper and the fact that all returned values are clipped into $[0,B]$.
Thus any returned random variable is bounded by $B$ in absolute value, which will be used both for concentration and for controlling
the effect of failure events.

\paragraph{I. $F^{\mathrm{OT}}_s$ satisfies cubic Taylor remainder (i.e., $\beta=3$ curvature).}
By \cref{lem:lipschitz-hess}, there exists a constant
\[
M = \mathcal{O}\left(\frac{1}{\tau^2\lambda^2}\right)\cdot \mathrm{poly}(K,\mu_{\min}^{-1})
\]
such that for all $Q,\hat Q\in\R^K$,
\begin{equation}
\label{eq:cubic-remainder-again}
\Big|F(Q)-T_2(Q;\hat Q)\Big|
\le
\frac{M}{6}\|Q-\hat Q\|_2^3.
\end{equation}
Equivalently, $F$ satisfies \cref{ass:curvature} with $\beta=3$ and $c_3=M/6$.
Consequently the tolerance schedule
\begin{equation}
\label{eq:rho-def-again}
\rho(\varepsilon):=\Big(\frac{6\varepsilon}{M}\Big)^{1/3}
\end{equation}
guarantees that if $\|Q_s-\hat Q_s\|_2\le \rho(\varepsilon)$ then the second-order Taylor remainder is at most~$\varepsilon$:
\[
|F(Q_s)-T_2(Q_s;\hat Q_s)|\le \varepsilon.
\]

\paragraph{II. A clean ``interface'' for \textsc{SampleV2}.}
We use \cref{lem:bias} to summarize what a single call to $\textsc{SampleV2}$ provides.

Fix any state $x$ and any $0<\eta<\kappa$.
Let $\hat Q_x$ be the internal baseline produced by $\textsc{SampleV2}(x,\eta,\cdot)$, and define the event
\[
\mathcal{G}_{x,\eta}:=\{\|Q_x-\hat Q_x\|_2\le \rho(\eta)\}.
\]
Then \cref{lem:bias} states that on $\mathcal{G}_{x,\eta}$,
\begin{equation}
\label{eq:sampleV2-cond-bias}
\Big|\E[\textsc{SampleV2}(x,\eta,\delta)\mid \hat Q_x] - V(x)\Big|
\le
\eta + \text{(recursion bias)}.
\end{equation}
Moreover, as in \citet{grill2019planning} and as noted right after \cref{lem:bias}, calling children with accuracy
$\eta/\sqrt{\gamma}$ yields a contraction on the recursion bias, and with an appropriate allocation of failure
probabilities across recursive calls (and because of clipping) the \emph{unconditional} bias becomes $\mathcal{O}(\eta)$.
Concretely, there exists a constant $c_{\mathrm{bias}}=c_{\mathrm{bias}}(K,\gamma,\lambda,\tau,\mu_{\min},\|C\|_\infty)$ such that
\begin{equation}
\label{eq:sampleV2-uncond-bias}
\Big|\E[\textsc{SampleV2}(x,\eta,\delta)]-V(x)\Big|
\le
c_{\mathrm{bias}}\eta,
\end{equation}
and since the output is clipped to $[0,B]$ we also have the uniform boundedness
\begin{equation}
\label{eq:sampleV2-bounded}
0\le \textsc{SampleV2}(x,\eta,\delta)\le B \qquad\text{a.s.}
\end{equation}
(Any additive polylogarithmic dependence on $1/\delta$ induced by the failure-probability bookkeeping
is absorbed by the $\tilde O(\cdot)$ notation used in the theorem statement.)

\paragraph{III. Correctness of \textsc{estimateQ} and of the top-level output.}
Algorithm~\ref{alg:top} returns $F(\hat Q_s)$ where $\hat Q_s:=\textsc{estimateQ}(s,\varepsilon,\delta/2)$.

\smallskip
\noindent\textbf{IIIa: The sampling model inside \textsc{estimateQ}.}
Fix an action $a\in\mathcal{A}$.
Each inner-loop iteration in Algorithm~\ref{alg:estimateq} draws $(R_i,Z_i)\leftarrow \textsc{Oracle}(s,a)$ and then sets
\[
\hat V_i := \textsc{SampleV2}(Z_i,\varepsilon/\sqrt{\gamma},\delta'),
\qquad
q_i := R_i + \gamma \hat V_i.
\]
Define the ``ideal'' one-step return
\[
q_i^\star := R_i + \gamma V(Z_i).
\]
Then $\E[q_i^\star]=Q_s(a)$ by definition of $Q_s$.
Using \eqref{eq:sampleV2-uncond-bias} with $\eta=\varepsilon/\sqrt{\gamma}$ and then multiplying by $\gamma$ yields
\begin{align}
\big|\E[q_i]-Q_s(a)\big|
&=
\gamma\big|\E[\hat V_i]-\E[V(Z_i)]\big|
\le
\gamma\cdot c_{\mathrm{bias}}\cdot \frac{\varepsilon}{\sqrt{\gamma}}
=
c_{\mathrm{bias}}\sqrt{\gamma}\varepsilon
\le
c_{\mathrm{bias}}\varepsilon.
\label{eq:q_i-bias}
\end{align}
Moreover, since $R_i\in[0,1]$ and $\hat V_i\in[0,B]$ by \eqref{eq:sampleV2-bounded},
\begin{equation}
\label{eq:q_i-bounded}
0\le q_i \le 1+\gamma B \qquad\text{a.s.}
\end{equation}

\smallskip
\noindent\textbf{IIIb: concentration of $\hat Q_s(a)$ around its mean.}
Let $N$ be the sample size in Algorithm~\ref{alg:estimateq},
\[
N=\Theta\big(\varepsilon^{-2}\log(2K/\delta)\big).
\]
Conditioned on the past, the $q_i$'s used for a fixed $(s,a)$ are i.i.d.\ bounded random variables
(we can treat them as independent because the algorithm uses fresh oracle calls and independent recursion randomness each time).
By Hoeffding's inequality and \eqref{eq:q_i-bounded},
\begin{equation}
\label{eq:hoeffding-q}
\Pr\Bigg(\Big|\hat Q_s(a)-\E[q_i]\Big|\ge t\Bigg)
\le
2\exp\left(-\frac{2Nt^2}{(1+\gamma B)^2}\right),
\end{equation}
where $\hat Q_s(a)=\frac{1}{N}\sum_{i=1}^N q_i$.

Choose $t=\varepsilon$ and take a union bound over $a\in\mathcal{A}$.
With the stated choice of $N$ (absorbing constants into the $\Theta(\cdot)$), we obtain an event $\mathcal{E}_Q$
with probability at least $1-\delta/2$ on which simultaneously for all actions,
\begin{equation}
\label{eq:Qhat-close-to-mean}
\big|\hat Q_s(a)-\E[q_i]\big| \le \varepsilon.
\end{equation}
Combining \eqref{eq:Qhat-close-to-mean} with the bias bound \eqref{eq:q_i-bias} gives that on $\mathcal{E}_Q$,
\begin{equation}
\label{eq:Qhat-close-to-Q}
\big|\hat Q_s(a)-Q_s(a)\big|
\le
\varepsilon + c_{\mathrm{bias}}\varepsilon
\le
c_Q\varepsilon,
\qquad
\text{for all }a,
\end{equation}
for some constant $c_Q\ge 1$.

Thus, on $\mathcal{E}_Q$,
\begin{equation}
\label{eq:infty-bound}
\|\hat Q_s-Q_s\|_\infty \le c_Q\varepsilon.
\end{equation}

\smallskip
\noindent\textbf{IIIc: Lipschitzness of $F$ implies value accuracy.}
By \cref{prop:grad-hess}, $\nabla F(Q)\in\Delta(\mathcal{A})$ for all $Q$.
A standard convex-analysis consequence is that $F$ is $1$-Lipschitz with respect to $\|\cdot\|_\infty$:
for any $Q,\hat Q\in\R^K$,
\begin{equation}
\label{eq:F-Lipschitz-infty}
|F(Q)-F(\hat Q)| \le \|Q-\hat Q\|_\infty.
\end{equation}
(Indeed, $F(Q)-F(\hat Q)=\int_0^1 \langle \nabla F(\hat Q+t(Q-\hat Q)), Q-\hat Q\rangle dt$ and
$\|\nabla F(\cdot)\|_1=1$.)

Applying \eqref{eq:F-Lipschitz-infty} with $Q=Q_s$ and $\hat Q=\hat Q_s$ and using \eqref{eq:infty-bound} yields on $\mathcal{E}_Q$,
\[
\big|F(\hat Q_s)-F(Q_s)\big|
\le
\|\hat Q_s-Q_s\|_\infty
\le
c_Q\varepsilon.
\]
Thus, with probability at least $1-\delta/2$, the output of Algorithm~\ref{alg:top} is $c_Q\varepsilon$-accurate.
Replacing $\varepsilon$ by $\varepsilon/c_Q$ (a constant-factor rescaling) yields the stated $\varepsilon$-accuracy,
so we henceforth treat $c_Q$ as absorbed into constants.

Finally, we also need to account for the failure probability of all recursive calls used within \textsc{estimateQ}.
This is handled by setting $\delta'$ in Algorithm~\ref{alg:estimateq} so that a union bound over all $K N$ internal
calls gives total failure probability at most $\delta/2$ (e.g.\ take $\delta'=\delta/(4KN)$, and then use the internal
$\delta/8,\delta/16,\dots$ splits in Algorithm~\ref{alg:samplev2}).
Together with $\Pr(\mathcal{E}_Q)\ge 1-\delta/2$, a union bound yields overall success probability at least $1-\delta$.

This completes the correctness part.

\paragraph{IV. Oracle-call complexity of \textsc{SampleV2}: a cascade recurrence.}
We now bound the number of oracle calls.
Let $T(\eta)$ denote the \emph{worst-case} number of oracle calls made by a single call
$\textsc{SampleV2}(x,\eta,\cdot)$, maximized over the input state $x$ and over all internal randomness.
(We suppress $\delta$ in $T(\cdot)$ because $\tilde O(\cdot)$ hides all polylogarithmic factors coming from confidence splitting.)

\smallskip
\noindent\textbf{Base and coarse regimes.}
If $\eta\ge B$, \textsc{SampleV2} returns immediately, so $T(\eta)=0$.
If $\eta\ge\kappa$, it calls $\textsc{estimateQ}(x,\eta,\cdot)$ once and then returns $F(\hat Q_x)$.
Since $\kappa$ is a constant independent of $\eta$ (depends only on $M$), this regime has $T(\eta)=\tilde O(1)$
and will not affect the asymptotic exponent as $\eta\downarrow 0$.

\smallskip
\noindent\textbf{Fine regime: $\eta<\kappa$.}
Inspect Algorithm~\ref{alg:samplev2}. The dominant work is the call
$\hat Q_x\leftarrow \textsc{estimateQ}(x,\rho(\eta),\cdot)$, where $\rho(\eta)=(6\eta/M)^{1/3}$.
By Algorithm~\ref{alg:estimateq}, this call performs
\[
N(\rho(\eta))=\Theta\big(\rho(\eta)^{-2}\log(\cdot)\big)
\]
samples per action, i.e.\ $\Theta(K\rho(\eta)^{-2})$ oracle calls at the current state.
Each such oracle call also triggers \emph{one} recursive call to \textsc{SampleV2} on the next state with tolerance
$\rho(\eta)/\sqrt{\gamma}$.
Thus, ignoring constants and polylogs,
\begin{equation}
\label{eq:T-main-term}
\text{cost of }\textsc{estimateQ}(x,\rho(\eta),\cdot)
\le
c_1\,K\rho(\eta)^{-2}\Big(1+T(\rho(\eta)/\sqrt{\gamma})\Big),
\end{equation}
for some constant $c_1$.

In addition, the linear and quadratic correction blocks in Algorithm~\ref{alg:samplev2}
use only a \emph{constant} number of extra oracle calls and a constant number of extra recursive calls at tolerance
$\eta/\sqrt{\gamma}$ (one call for the linear term and a constant number for the quadratic term).
Therefore there exists a constant $c_2$ such that the fine-regime cost satisfies
\begin{equation}
\label{eq:T-recurrence-full}
T(\eta)
\le
c_1\,K\rho(\eta)^{-2}\Big(1+T(\rho(\eta)/\sqrt{\gamma})\Big)
+
c_2\Big(1+T(\eta/\sqrt{\gamma})\Big)
+
c_2.
\end{equation}
Now use $\rho(\eta)=\Theta(\eta^{1/3})$, so $\rho(\eta)^{-2}=\Theta(\eta^{-2/3})$.
Also define the constant
\[
c:=\frac{1}{\sqrt{\gamma}}\Big(\frac{6}{M}\Big)^{1/3},
\qquad\text{so that}\qquad
\frac{\rho(\eta)}{\sqrt{\gamma}} = c\eta^{1/3}.
\]
Absorbing constants and polylogarithms into $\tilde O(\cdot)$, \eqref{eq:T-recurrence-full} implies the simplified cascade recurrence
\begin{equation}
\label{eq:T-cascade}
T(\eta)
\le
\tilde O\big(\eta^{-2/3}\big)\cdot \Big(1+T(c\eta^{1/3})\Big)
+
\tilde O\big(1+T(\eta/\sqrt{\gamma})\big).
\end{equation}
The second term does not change the exponent because it appears with only constant multiplicity; it can be absorbed into constants
once we know $T(\eta)$ grows polynomially as $\eta\downarrow 0$.

\paragraph{V. Solve the cascade recurrence: $T(\eta)=\tilde O(\eta^{-1})$.}
We now show that \eqref{eq:T-cascade} yields
\begin{equation}
\label{eq:T-solution}
T(\eta)=\tilde O(\eta^{-1}).
\end{equation}

To see the exponent, first ignore the lower-order $\tilde O(1+T(\eta/\sqrt{\gamma}))$ term and focus on the dominant cascade:
\begin{equation}
\label{eq:T-cascade-dominant}
T(\eta)\ \lesssim\ A\eta^{-2/3}\,T(c\eta^{1/3}) + A\eta^{-2/3},
\end{equation}
for some constant $A=\tilde O(K)$.
Define the tolerance sequence $\eta_0:=\eta$ and $\eta_{t+1}:=c\eta_t^{1/3}$.
As in \citet{grill2019planning} and in the proof of \cref{thm:beta-tradeoff}, this sequence reaches a constant in
$O(\log\log(1/\eta))$ steps because $1/3<1$.
Let $L$ be the smallest index with $\eta_L\ge \eta_{\mathrm{base}}$ for a fixed constant $\eta_{\mathrm{base}}\in(0,\kappa]$.
Then $T(\eta_L)=\tilde O(1)$.

Unrolling \eqref{eq:T-cascade-dominant} along the sequence gives
\[
T(\eta_0)
\ \lesssim\
A^L\Big(\prod_{t=0}^{L-1}\eta_t^{-2/3}\Big)\,T(\eta_L)
+
\sum_{j=0}^{L-1} A^{j+1}\Big(\prod_{t=0}^{j}\eta_t^{-2/3}\Big).
\]
The sum is dominated by its final term up to a multiplicative polylog factor (because the products grow rapidly as $\eta_t$ decreases backward),
so it suffices to understand the product $\prod_{t=0}^{L-1}\eta_t^{-2/3}$.

Using the explicit form
\[
\eta_t
=
c^{\,1+1/3+\cdots+1/3^{t-1}}\ \eta^{1/3^{t}}
=
c^{\frac{1-(1/3)^t}{1-1/3}}\ \eta^{1/3^{t}},
\]
we have
\[
\eta_t^{-2/3}
=
c^{-\Theta(1)}\ \eta^{-2/3^{t+1}}.
\]
Therefore
\[
\prod_{t=0}^{L-1}\eta_t^{-2/3}
=
c^{-\Theta(L)}\ \eta^{-\sum_{t=0}^{L-1}2/3^{t+1}}
=
c^{-\Theta(L)}\ \eta^{-(1-(1/3)^L)}.
\]
Since $(1/3)^L$ is negligible and $L=O(\log\log(1/\eta))$, the factor $c^{\Theta(L)}A^L$ is polylogarithmic in $1/\eta$,
and the exponent of $\eta$ is $1$ up to a vanishing correction.
This yields \eqref{eq:T-solution}, i.e.\ $T(\eta)=\tilde O(\eta^{-1})$.

(If desired, one can incorporate the suppressed $\tilde O(1+T(\eta/\sqrt{\gamma}))$ term by induction, noting that
$\eta/\sqrt{\gamma}$ differs from $\eta$ only by a constant factor and thus preserves the same $\eta^{-1}$ scaling.)

\paragraph{VI. Total oracle calls of \textsc{SecondOrderSmoothCruiser}.}
Algorithm~\ref{alg:top} performs a single call to $\textsc{estimateQ}(s,\varepsilon,\delta/2)$ and then applies $F$.
The application of $F$ is computational (no oracle calls), so the oracle complexity is exactly that of \textsc{estimateQ}.

By Algorithm~\ref{alg:estimateq}, \textsc{estimateQ} performs $K\,N(\varepsilon,\delta/2)$ oracle calls at $(s,a)$,
where $N(\varepsilon,\delta/2)=\Theta(\varepsilon^{-2}\log(2K/\delta))$.
Each oracle call also triggers one call to \textsc{SampleV2} at tolerance $\varepsilon/\sqrt{\gamma}$.
Hence the total oracle calls satisfy
\begin{align*}
n(\varepsilon,\delta)
&\le
K\,N(\varepsilon,\delta/2)\cdot \Big(1+T(\varepsilon/\sqrt{\gamma})\Big)\\
&=
\tilde O\big(\varepsilon^{-2}\big)\cdot \Big(1+\tilde O\big((\varepsilon/\sqrt{\gamma})^{-1}\big)\Big)
\qquad\text{(by \eqref{eq:T-solution})}\\
&=
\tilde O\big(\varepsilon^{-2}\big)\cdot \tilde O\big(\varepsilon^{-1}\big)
=
\tilde O(\varepsilon^{-3}).
\end{align*}
All constants depend only on $(K,\gamma,\lambda,\tau,\mu_{\min},\|C\|_\infty)$ through $B$, $M$ and the fixed multiplicative
costs in the algorithms, completing the proof.
\end{proof}

\end{document}